\documentclass{article} 

\usepackage{graphicx} 
\usepackage[utf8]{inputenc}
\usepackage{pgfplots,pgfplotstable}
    \usepgfplotslibrary{fillbetween}
    \pgfplotsset{compat=1.17}
\usepackage{tikz}
    \usetikzlibrary{shapes,decorations}
    \usetikzlibrary{positioning}
    \usetikzlibrary{external} 
\usepackage[ruled, noline,noend]{algorithm2e} 
\usepackage{wrapfig} 
\usepackage{amsmath, amsthm, mathrsfs, mathtools}
\usepackage{amssymb}
\usepackage{fullpage}
\usepackage{enumitem}
\usepackage{authblk}
\usepackage{hyperref}
\usepackage[numbers]{natbib}
\usepackage[capitalize]{cleveref} 

\newcommand{\cD}{\mathcal{D}}

\newcommand{\E}{\mathbb{E}}

\newcommand{\I}{\mathbb{I}}

\newcommand{\N}{\mathbb{N}}

\newcommand{\Pb}{\mathbb{P}}

\newcommand{\R}{\mathbb{R}}

\newcommand{\e}{\varepsilon}

\newcommand{\rmd}{\mathrm{d}}
\newcommand{\dif}{\,\rmd}
\DeclareMathOperator*{\argmax}{argmax}
\newcommand{\lrb}[1]{\left(#1\right)}
\newcommand{\brb}[1]{\bigl(#1\bigr)}
\newcommand{\Brb}[1]{\Bigl(#1\Bigr)}

\newcommand{\lsb}[1]{\left[#1\right]}
\newcommand{\bsb}[1]{\bigl[#1\bigr]}
\newcommand{\Bsb}[1]{\Bigl[#1\Bigr]}
\newcommand{\bbsb}[1]{\biggl[#1\biggr]}
\newcommand{\lcb}[1]{\left\{#1\right\}}
\newcommand{\bcb}[1]{\bigl\{#1\bigr\}}

\newcommand{\bce}[1]{\bigl\lceil#1\bigr\rceil}

\newcommand{\labs}[1]{\left\lvert#1\right\rvert}
\newcommand{\babs}[1]{\bigl\lvert#1\bigr\rvert}

\newcommand{\m}{\setminus}
\newcommand{\iop}{\infty}

\newcommand{\fracc}[2]{#1/#2}
\newcommand{\var}{\mathrm{Var}}
\newcommand{\gft}{\mathrm{g}}
\newcommand{\GFT}{\mathrm{GFT}}
\newcommand{\rhot}{\widetilde \rho}
\newcommand{\ftr}{FT$\rho$}
\newcommand{\ftmtr}{FTMT$\rho$}
\DeclareSymbolFont{extraup}{U}{zavm}{m}{n}
\DeclareMathSymbol{\clubsuit}{\mathalpha}{extraup}{84}
\DeclareMathSymbol{\spadesuit}{\mathalpha}{extraup}{81}
\DeclareMathSymbol{\varheartsuit}{\mathalpha}{extraup}{86}
\DeclareMathSymbol{\vardiamondsuit}{\mathalpha}{extraup}{87}

\newtheorem{lemma}{Lemma}
\newtheorem{theorem}{Theorem}
\theoremstyle{definition}

\title{An Online Learning Theory of Brokerage}

\author[1]{Nata\v{s}a Boli\'{c}}
\author[1]{Tommaso Cesari}
\author[2,3]{Roberto Colomboni}

\affil[1]{University of Ottawa, Ottawa, Canada}
\affil[2]{Istituto Italiano di Tecnologia, Genova, Italy}
\affil[3]{Universit\`a degli Studi di Milano, Milano, Italy}

\begin{document}

\maketitle

\begin{abstract}
\noindent We investigate brokerage between traders from an online learning perspective.
At any round $t$, two traders arrive with their private valuations, and the broker proposes a trading price.
Unlike other bilateral trade problems already studied in the online learning literature, we focus on the case where there are no designated buyer and seller roles: each trader will attempt to either buy or sell depending on the current price of the good.

We assume the agents' valuations are drawn i.i.d.\ from a fixed but unknown distribution.
If the distribution admits a density bounded by some constant $M$, then, for any time horizon $T$:
 \begin{itemize}
	\item If the agents' valuations are revealed after each interaction, we provide an algorithm achieving regret $M \log T$ and show this rate is optimal, up to constant factors.
	\item If only their willingness to sell or buy at the proposed price is revealed after each interaction, we provide an algorithm achieving regret $\sqrt{M T}$ and show this rate is optimal, up to constant factors.
\end{itemize}
Finally, if we drop the bounded density assumption, we show that the optimal rate degrades to $\sqrt{T}$ in the first case, and the problem becomes unlearnable in the second.

\end{abstract}


\section{Introduction}

Over-the-counter (OTC) markets offer a variety of decentralized alternatives to traditional financial exchanges and have gained prominence for their flexibility, diversity, and accessibility for participants.
Paraphrasing the words of Tolstoy, ``\emph{All centralized markets are the same, but each OTC market is unique in its own way}'' \cite{lucas1989effects,weill2020search}.
In recent years, OTC markets have flourished, becoming an indispensable part of the global financial ecosystem, with a steady growth trend documented since 2016 \cite{bis2023} and the value of domestic assets traded in OTC markets surpassing a staggering 50,000 billion USD (exceeding centralized markets by over 20,000 billion USD) in 2020 \cite{weill2020search}.
Central to the functioning of decentralized OTC markets are brokers who, acting as intermediaries, bridge the gap between buyers and sellers, ensuring that trades are executed smoothly. 
Beyond mere intermediation, brokers play a significant role in price discovery, gauging demand and supply to determine optimal asset prices. 
However, the classical impossibility result of Myerson and Satterthwaite \cite{myerson1983efficient} highlights that the role of the broker is not without challenges. 
Inspired by a recent stream of literature \cite{cesa2021regret, azar2022alpha, cesa2023bilateral, cesa2023repeated}, we approach the bilateral trade problem of brokerage between traders through the lens of online learning. 
When viewed from a regret minimization perspective, bilateral trade has been explored over rounds of seller/buyer interactions with no prior knowledge of their private valuations, but only under rigid buyer and seller roles.
In contrast, it's important to note that in many key OTC markets, traders are willing to either buy or sell, depending on the prevailing market conditions \cite{sherstyuk2020randomized}. 
These markets encompass a wide array of asset trades, including stocks, derivatives, art, collectibles, precious metals and minerals, energy commodities like gas and oil, as well as digital currencies (cryptocurrencies), among others.

Motivated by brokerage between traders in these markets, we aim to fill the gap in the online learning literature on bilateral trade, examining scenarios where traders' roles as buyers or sellers are not strictly defined.

\subsection{Setting}
We study the following problem.
At each time $t\in \N$,
\begin{enumerate}
    \item Two traders arrive with private valuations $V_{2t-1}$ and $V_{2t}$.
    \item The broker proposes a trading price $P_t$.
    \item If the price $P_t$ falls between the lowest\footnote{We denote the minimum (resp., maximum) of any two real numbers $x,y\in \R$ by $x\wedge y$ (resp., $x\vee y$).} $V_{2t-1} \wedge V_{2t}$ and highest $V_{2t-1} \vee V_{2t}$  valuations (i.e., if the trader with the smallest valuation is eager to sell at price $P_t$ and the other is willing to buy at $P_t$), the trader with the highest valuation buys the item from the trader with the lowest valuation paying the brokerage price $P_t$. 
    \item Some feedback is revealed.
\end{enumerate}
Consistently with the existing bilateral trade literature, we assume valuations and prices belong to $[0,1]$, and the reward associated with each interaction is the sum of the utilities of the traders, known as \emph{gain from trade}. 
Formally, for any $p,v_1,v_2 \in [0,1]$, the gain from trade of a price $p$ when the valuations of the traders are $v_1$ and $v_2$ is
\[
	\gft(p,v_1,v_2) \coloneqq \lrb{ v_1 \vee v_2 - v_1 \wedge v_2 } \I \lcb{ v_1 \wedge v_2 \le p \le v_1 \vee v_2 } \;.
\]
The aim of the learner is to minimize the \emph{regret}, defined, for any time horizon $T\in\N$, as
\[
	R_T \coloneqq \sup_{p\in[0,1]} \E \lsb{ \sum_{t=1}^T \GFT_t(p) } - \E \lsb{ \sum_{t=1}^T \GFT_t(P_t) } \;,
\]
where we let $\GFT_t(q) \coloneqq \gft(q, V_{2t-1}, V_t)$ for all $q\in[0,1]$ and the expectations are taken with respect to the (possible) randomness of $(V_t)_{t\in \N}$ and $(P_t)_{t\in \N}$.

We study this problem under the assumption that the traders' valuations $V, V_1, V_2,\dots$ are generated i.i.d.\ from an unknown distribution $\nu$, a natural assumption for large and stable markets.

Finally, we consider two different types of feedback:

\begin{itemize}
    \item \emph{Full feedback.} The valuations $V_{2t-1}$ and $V_{2t}$ of the two current traders are revealed at the end of every round $t$.
    \item \emph{Two-bit feedback.} Only the indicator functions $\I\{P_t \le V_{2t-1}\}$ and  $\I\{P_t \le V_{2t}\}$ are revealed at the end of every round $t$.
\end{itemize}
The information collected by the full feedback model corresponds to \emph{direct revelation mechanisms}, where the traders communicate their valuations $V_{2t-1}$ and $V_{2t}$ before each round, and the price proposed by the mechanism at time $t$ only depends on past bids $V_1, \dots, V_{2t-2}$.
The two-bit feedback model corresponds to \emph{posted price} mechanisms, where traders only communicate their willingness to buy or sell at the posted price,\footnote{For the nitpicker, the natural feedback model in this case would be to reveal the four bits $\I\{P_t<V_{2t-1}\}$, $\I\{P_t=V_{2t-1}\}$, $\I\{P_t<V_{2t}\}$, and  $\I\{P_t = V_{2t}\}$. 
This feedback is more informative than the one we propose; hence our upper bounds hold \emph{a fortiori} in the four-bit feedback model. For the lower bounds, in the bounded density case, given that the two and four-bit feedback give the same information with probability $1$, the same results hold; in the general case, a straightforward adaptation of the exact same construction in our lower bounds gives the result for the four-bit feedback case. 
Given this equivalency between these two models, we opt for the two-bit feedback for the sake of conciseness.} and the valuations $V_{2t-1}$ and $V_{2t}$ are \emph{never} revealed.

\subsection{Overview of Our Contributions} 
If the distribution $\nu$ of the traders' valuations admits a density bounded by some constant $M$, then, for any time horizon $T$:
\begin{itemize}
    \item In the full feedback case, we design an algorithm (\Cref{a:ftm}) achieving regret $O \lrb{ M \log T }$ (\Cref{t:ftm}) and provide a matching lower bound $\Omega \lrb{ M \log T }$ (\Cref{t:lower-bound-full-M}).
    \item In the two-bit feedback case, we design an algorithm (\Cref{a:etc}) achieving regret $ O \brb{ \sqrt{M T} }$ (\Cref{t:ETC}) and provide a matching lower bound $\Omega \brb{ \sqrt{M T} }$ (\Cref{t:lower-bound-two-bit}). 
\end{itemize}
If we drop the bounded density assumption, we show that the optimal rate degrades to $\Theta\brb{ \sqrt{T} }$ in the full feedback case (\Cref{t:ftr} and \ref{t:lower-bound-full-general}), while the problem becomes unlearnable in the two-bit feedback case (\Cref{t:lower-bound-two-bit-general}).
Furthermore, in the full feedback case, we design an algorithm (\Cref{a:ftmtr}) achieving simultaneously the optimal $O(\sqrt{T})$ regret in the general case and $O(M\log T)$ in the bounded density case while being oblivious to the validity of the bounded density assumption.

Finally, we stress that ours is the first paper on online learning in bilateral trade where lower bounds have the correct dependency on $M$. In all existing literature, optimality was only proved in the time horizon $T$, not in $M$.

\subsection{Techniques and Challenges}
\label{s:tech-and-challenges}

The two feedback models we consider present different challenges.

\paragraph{Full Feedback.}
It can be proven that assuming $\nu$ has a bounded density implies the Lipschitzness of the expected gain from trade. 
The standard approach for Lipschitz objectives is to discretize the domain uniformly and run an optimal expert algorithm on the discretization, which yields straightforwardly an $\frac{M}{K}T +\sqrt{T \log K}$ regret guarantee, where $K$ is the number of prices in the discretization and $T$ is the time horizon.
Alternatively, one could build a reduction to known bilateral trade problems to obtain an improved $\sqrt{T}$ bound (for further details, see the Related Work section).
Both these natural choices are highly suboptimal.
In contrast, we exploit the specific structure of our problem to achieve an exponential gain, resulting in an $M\log T$ regret bound.
Regarding the lower bound, it's worth noting that determining the correct dependence on $M$ in such problems is remarkably challenging.
In fact, no previous papers on online learning for bilateral trade problems showcased lower bounds displaying \emph{any} dependence on $M$, let alone the correct one.
We manage to achieve this through two pivotal lemmas: an Approximation (\Cref{l:key}) and a Representation (\Cref{l:repre}) lemma.
These two technical results lead us to \Cref{t:mysticus}, which establishes two points: first, $\E[V]$ is always a maximizer of the expected gain from trade; second, posting a price close to the maximizer has a cost only quadratic in the distance.
These two facts point to a peculiar similarity between ours and the statistical problem of estimating an expectation online with a quadratic loss, an observation that turns out to be crucial in a setting where we could not otherwise recycle any of the existing techniques of other online learning settings in bilateral trade (none of which features logarithmic rates).
Still, following this path is made technically challenging by the fact that building a hard instance (in a setting where we cannot control the gain from trade directly but only indirectly through the distributions of traders' valuations) is far from being straightforward.
We manage to circumvent this roadblock by carefully designing a $1$-parameter family of hard distributions which are used to build a reduction to a Bayesian problem. By means of non-trivial information-theoretic and probabilistic arguments, we can exploit the quadratic loss intuition, to obtain a lower bound featuring the optimal $M \log T$ rate.

Beyond this, the Approximation and Representation Lemmas also suggest the best pathway to tackle the non-bounded-density setting. 
The idea of trying to approximate the representation given by the second lemma and the fact that this representation reduces to a simpler form in the bounded-density case allows us to design a simple algorithm that enjoys optimal regret guarantees both in the bounded-density ($M\log T$ regret) and the non-bounded density ($\sqrt T$ regret) cases while being completely oblivious to which of the two assumptions hold.

\paragraph{Two-Bit Feedback}
A clear challenge of the two-bit feedback is posed by its low quality: at each time step, this feedback is not even sufficient to reconstruct the so-called bandit feedback (i.e., the reward $\GFT_t(P_t)$ of the posted price $P_t$), which is typically considered the bare minimum in online learning.
Once more, the Approximation and Representation Lemmas point to the strategy of estimating $\E[V]$ with the available feedback.
Leveraging a discretization argument (\cref{l:approx}) and concentration inequalities, we obtain an upper bound of $\sqrt{MT}$ in the bounded-density case.
To show that this rate is optimal, we build on the hard instances we designed for the full feedback case. 
The difference now is that the scarcity of feedback leads to a setting similar to the so-called ``revealing action'' problem \cite{cesa2006prediction}.
A notable difference is that the regular revealing action problem has an optimal $T^{2/3}$ regret, but, cast in our problem, due to the shape of the reward around the maximizer, a careful analysis shows that the regret of the adapted revealing action is $\sqrt{ M T}$.
Again, we stress that this is the first lower bound in online learning with two-bit feedback for bilateral trade with any dependence on $M$, let alone the correct one.

As for the full feedback case, we show how the problem changes without the bounded density assumption.
Dropping the assumption leads to a pathological phenomenon typical of bilateral trade problems (known as needle-in-a-haystack) leading us to the design of hard instances in which all-but-one prices suffer a high $\Omega(1)$-regret, and where it is essentially impossible to find the optimal price among the continuum amount of suboptimal ones given the small amount of information carried by the two-bit feedback.

\subsection{Related Work}
\label{s:related}

Since its inception in the seminal work of Myerson and Satterthwaite \cite{myerson1983efficient}, a large body of literature on bilateral trade has emerged, mainly from a best-approximation/game-theoretic perspective \cite{Colini-Baldeschi16,Colini-Baldeschi17,BlumrosenM16,brustle2017approximating,colini2020approximately,babaioff2020bulow,dutting2021efficient,DengMSW21,kang2022fixed,archbold2023non}. For a discussion about this literature, see \cite{cesa2023bilateral}.

In recent years, bilateral trade has also been studied in online learning settings. 
Given that these works are the most closely related to ours, we focus on discussing our relationship with them.

In \cite{cesa2021regret}, the authors studied a bilateral trade setting where the sequence of seller/buyer valuations $(S_t, B_t)_{t\in\N}$ is an i.i.d.\ process. 
This is also the case in our setting, via the reduction where we set $S_t \coloneqq V_{2t-1}\wedge V_{2t}$ and $B_t \coloneqq V_{2t-1}\vee V_{2t}$.
In the full-feedback case, they obtain regret $\widetilde O\brb{\sqrt T}$ (improved to $O\brb{\sqrt{T}}$ in \cite{cesa2023bilateral}) and show that any algorithm suffers worst-case\footnote{\label{footnote}Among all i.i.d.\ sequences $(S_t,B_t)_{t \in \N}$, not just those arising from the reduction $S_t \coloneqq V_{2t-1}\wedge V_{2t}$ and $B_t \coloneqq V_{2t-1}\vee V_{2t}$.} regret $\Omega\brb{\sqrt{T}}$, even under the bounded density assumption (shorthanded by BDA, in the rest of this section), i.e., when the joint distribution of seller/buyer valuations has a density bounded by some constant $M$.
In contrast, in the full feedback case, we prove that this rate is suboptimal in our setting under BDA, and can be improved exponentially to an optimal $\Theta (M \log T)$ (\Cref{t:ftm} and \ref{t:lower-bound-full-M}).
Without BDA, one could run their FBP algorithm in our setting, keeping its $O \brb{ \sqrt{T} }$ guarantees. 
Instead, we elected to propose a different algorithm (\Cref{a:ftr}) with a simpler analysis, achieving a regret with the same dependence on $T$ but with better numerical constants. 
We also remark that a new lower bound proof is required to prove this rate is optimal (\Cref{t:lower-bound-full-general}), given that the family of instances to prove the lower bound in \cite{cesa2021regret} cannot arise via the reduction above.

In the two-bit feedback i.i.d.\ setting, \cite{cesa2021regret} show that any algorithm has to suffer linear worst-case$^\text{\ref{footnote}}$ regret, even under BDA. 
In contrast, we prove that in the two-bit feedback case, our problem is learnable under BDA and obtain optimal $\Theta\brb{ \sqrt{M T}}$ guarantees (\Cref{t:ETC} and \ref{t:lower-bound-two-bit}).
Without BDA, we show that our problem is also unlearnable (\Cref{t:lower-bound-two-bit-general}).

Under BDA and the additional assumption that $S_t$ and $B_t$ are independent of each other (which in not the case in our problem because of the correlation between the maximum and the minimum of the traders' valuations), \cite{cesa2021regret} achieve a regret $\widetilde O ( M^{1/3}T^{2/3})$  (improved to $O ( M^{1/3}T^{2/3})$ in \cite{cesa2023bilateral}) and show that any algorithm suffers worst-case regret $\Omega(T^{2/3})$ when $M$ is sufficiently large.

The bilateral trade problem has also been studied under a \emph{weak budget balance} assumption in \cite{cesa2023bilateral}. 
In this setting, the learner can post two different prices: a selling price $p$ (to the seller) and a buying price $q$ (to the buyer), with $p\le q$.
The authors prove that this bilateral trade problem is learnable even without assuming the independence of seller's $S_t$ and buyer's $B_t$ valuations, providing an $\widetilde O ( M T^{3/4} )$ upper bound. 
\cite{cesa2023repeated} prove that this rate is optimal in $T$ (up to logarithmic terms), providing a matching (in $T$) $\Omega( T^{3/4} )$ lower bound.
We remark that even if the broker were permitted to offer two distinct prices, $p\le q$ (with $p$ as the selling price and $q$ as the buying price), our results would still be optimal. This is because there's no reason to do so in the full feedback scenario\footnote{As \cite{cesa2023repeated} remark, ``The only reason for a budget-balanced algorithm to post two different prices is to
obtain more information. A direct verification shows that the expected gain from trade can always be
maximized by posting the same price to both the seller and the buyer''}, and our two-bit case lower bound remains valid with slight adjustments.

In the adversarial case, \cite{cesa2021regret} show that learning is impossible. Their lower bound construction yields that learning in the adversarial case is also impossible in our setting. In fact, given that in the adversarial case the sequence of traders' valuations can be chosen arbitrarily, we can set $V_{2t-1}\coloneqq s_t$ and $V_{2t} \coloneqq b_t$ where $s_t$ and $b_t$ are defined as in the adversarial lower bound construction in the proof of Theorem 5.1 in \cite{cesa2021regret}, and the same proof applies verbatim.

To achieve learnability in the adversarial case, \cite{azar2022alpha} weakened the notion of regret, setting as a benchmark a multiple of (and not exactly the) the cumulative reward of the best fixed price in hindsight. In the full-feedback case they obtained $\widetilde{\Theta}(\sqrt{T})$ guarantees on the $2$-regret, while in the two-bit-feedback case they obtained a mismatching upper $\widetilde{O}(T^{3/4})$ and lower $\Omega(T^{2/3})$ bounds on the $2$-regret (and closing this gap is still an open problem in the bilateral trade literature). Analogous considerations such as the ones made above for the standard adversarial setting apply also to this weakened notion of regret in our setting.

As a final note on the current online learning literature on bilateral trade, we remind the reader again that ours is the first paper in this literature that has the correct dependency on both $M$ and $T$.

\section{Structural results}

We denote the Dirac measure based at $x\in\R$ by $\delta_x$, i.e., $\delta_x$ is the measure defined via the equation $\delta_x[A]=\I\{x \in A\}$ for any  set $A$. 
For any (signed) measure $\mu$ and any measurable set $E$, we will write $\mu E$ rather than $\mu[E]$ whenever this does not cause confusion.
For any measure $\mu$ over $[0,1]$, we let $\bar\mu \coloneqq \int_{[0,1]} x \dif \mu (x)$, and we define the functions $\rhot (\mu)$ and $\rho (\mu)$, for all $p\in[0,1]$,  by
\begin{align*}
        \rhot (\mu)(p) 
    & 
        \coloneqq \int_0^p \brb{ \mu [0,\lambda] + \mu [0,\lambda) } \dif \lambda + \brb{ \mu [0,p] + \mu [0,p) } (\bar \mu - p) \;,
    \\
        \rho (\mu)(p) 
    &
        \coloneqq \rhot(\mu)(p) + \mu\{p\} \lrb{ \int_0^p \mu[0,\lambda]\dif \lambda + \int_p^1 \mu[\lambda,1]\dif \lambda} \;.
\end{align*}
The following lemma shows that $\bar\mu$ maximizes $\rhot(\mu)$ (in general) and $\rho(\mu)$ (if $\mu$ has a bounded density), and the cost of approximating $\bar\mu$.  
\begin{lemma}[Approximation]
\label{l:key}
    If $\mu$ is a probability measure on $[0,1]$, then 
    $
        \rhot(\mu)(\bar\mu)
    =
        \max_{p \in [0,1]}\rhot(\mu)(p)
    $
    and, for any $p\in[0,1]$, $\rhot(\mu)(\bar\mu)-\rhot(\mu)(p) \le 2 |\bar\mu-p|$.
    If $\mu$ has a density bounded by $M>0$, then $\rho(\mu) = \rhot(\mu)$ and
    \[
        0 \le \rho(\mu)(\bar\mu)-\rho(\mu)(p) \le M \labs{ \bar \mu - p }^2 \;, \qquad\forall p\in[0,1] \;.
    \]
\end{lemma}

\begin{proof}
For $\lambda\in[0,1]$, let $m(\lambda) \coloneqq \mu[0,\lambda] + \mu[0,\lambda)$ and note that $m$ is a $[0,1]$-valued non-decreasing function of $\lambda$.
For any $p\in[0,1]$,
\begin{equation}
\label{misticanza}
    \rhot(\mu)(\bar\mu) - \rhot(\mu)(p)
=
    \int_p^{\bar\mu} \brb{ m(\lambda)- m(p)} \dif \lambda 
\quad
\begin{cases*}
    \ge 0 \;,\\
    \le 2|\bar\mu - p|\;,
\end{cases*}
\end{equation}
which implies that
$
        \rhot(\mu)(\bar\mu)
    =
        \max_{p \in [0,1]} \rhot(\mu)(p)
$.
Next, note that for all $p\in[0,1]$, $\babs{ \rhot(\mu)(p) - \rho(\mu)(p) }\le \mu\{p\}$, which, if $\mu$ has a density $f$ bounded by a constant $M$, implies $\rhot(\mu)(p) = \rho(\mu)(p)$ and
\begin{align*}
&
    \rho(\mu)(\bar\mu) - \rho(\mu)(p)
=
    \rhot(\mu)(\bar\mu) - \rhot(\mu)(p)
\overset{\eqref{misticanza}}{=}
    \int_p^{\bar\mu}\brb{m(\lambda)-m(p)} \dif \lambda
\\
&
\quad
=
    2 \int_p^{\bar\mu} \int_p^\lambda f(x)\dif x \dif \lambda
\le
    2M \labs{ \int_p^{\bar\mu} \labs{ \lambda-p } \dif\lambda }
=
    M \labs{\bar\mu-p}^2 \;. \qedhere
\end{align*}
    
\end{proof}%
The next lemma provides a crucial representation of the objective $p\mapsto\E\bsb{\GFT_t(p)}$.
Its long and (somewhat) tedious proof is deferred to the supplementary material.
\begin{lemma}[Representation]
\label{l:repre}
    For any $t\in \N$ and $p \in [0,1]$,
\[
    \E\bsb{\GFT_t(p)} = \rho(\nu)(p) \;.
\]
\end{lemma}
The following is an immediate corollary of \Cref{l:repre,l:key}.
\begin{theorem}
\label{t:mysticus}
    If $\nu$ admits a density bounded by some constant $M$, then for any $t \in \N$ and any $p \in [0,1]$, it holds that
\[
    0\le\E\bsb{\GFT_t\brb{\E[V]}}-\E\bsb{\GFT_t(p)} 
    \le M \cdot \babs{\E[V]-p}^2 \;,
\]
    and, in particular, $\max_{p \in [0,1]} \E\bsb{\GFT_t(p)} = \E\bsb{\GFT_t\brb{\E[V]}}$.
\end{theorem}
The previous theorem gives much intuition on the problem under the bounded density assumption.
It proves that the optimal action would be to post the (unknown) expected value $\E[V]$ of the valuations. 
Moreover, it suggests the strategy of approximating this value on the basis of the observed feedback, since posting a price close to expectation has only a quadratic cost in the approximation.

\section{Full Feedback}

We begin by studying the full feedback case (corresponding to direct revelation mechanisms) under the bounded density assumption.

\subsection{Upper Bound}
Following the intuition provided by \Cref{t:mysticus}, we introduce the Follow-the-Mean algorithm (FTM), which simply posts the empirical average of the past valuations (\Cref{a:ftm}).
\begin{algorithm}
Post $P_1 \coloneqq 1/2$, then receive feedback $V_{1}$, $V_{2}$\;
\For
{%
    time $t=2,3,\ldots$
}
{
    Post $P_t \coloneqq \frac{\sum_{s=1}^{2(t-1)}V_s}{2(t-1)}$, then receive feedback $V_{2t-1}$, $V_{2t}$\;
}
\caption{Follow the Mean (FTM)}
\label{a:ftm}
\end{algorithm}
The next theorem shows that FTM enjoys an $M\log T$ regret.
\begin{theorem}
\label{t:ftm}
If $\nu$ has density bounded by some constant $M>0$, then the regret of FTM satisfies, for all time horizons $T\ge 2$
\[
    R_T
\le
    \frac12 + \frac{M}{4}\brb{1 + \ln (T-1)} \;.
\]
\end{theorem}
{
\allowdisplaybreaks
\begin{proof}
For all time horizons $T\ge 2$, we have
\begin{align*}
&
R_{T} - \frac12
\le
    \sum_{t=2}^{T}
    \Brb{
        \E\bsb{\GFT_t\brb{\E[V]}}
        -
        \E\bsb{\GFT_t(P_t)}
    }
\\
&
\quad
\overset{\mathrm{(l)}}{=}
    \sum_{t=2}^{T}
    \E\bbsb{
        \Bsb{
        \E\bsb{\GFT_t\brb{\E[V]}
    -
        \GFT_t(p)}
        }_{p=P_t}
    }
\\
&
\quad
\overset{(\mathrm{t})}{\le}
    \sum_{t=2}^{T}\E\bbsb{\Bsb{M\babs{p-\E[V]}^2}_{p=P_t}}
=
    M\sum_{t=2}^{T}\E\Bsb{\babs{P_t-\E[V]}^2}
\\
&
\quad
\overset{(\mathrm{f})}{=}
    M\sum_{t=2}^{T}
    \int_0^\infty\Pb\Bsb{\babs{P_t-\E[V]}^2\ge\varepsilon} \dif\varepsilon
\overset{(\mathrm{h})}{\le}
    M\sum_{t=1}^{T-1}
    \int_0^\infty 2e^{-8t\varepsilon} \dif\varepsilon
\\
&
\quad
\overset{\phantom{(\mathrm{h})}}{=}
    \frac{M}{4}\sum_{t=1}^{T-1}
    \frac{1}{t}
\le
    \frac{M}{4}\lrb{1+
    \int_{1}^{T-1}\frac{1}{s}\dif s}
\le
    \frac{M}{4}\brb{1+\ln (T-1)} \;,
\end{align*}
where $\mathrm{(l)}$ follows from the Freezing Lemma (see, e.g., \cite[Lemma~8]{cesari2021nearest}) after observing that $\GFT_t(P_t) = \gft(P_t,V_{2t-1},V_{2t})$ and $P_t$ is independent of $(V_{2t-1},V_{2t})$; $\mathrm{(t)}$ from \Cref{t:mysticus}; $\mathrm{(f)}$ follows from Fubini's Theorem; and $\mathrm{(h)}$ from Hoeffding's inequality.
\end{proof}
}

\subsection{Lower Bound}
In this section, we prove the optimality of FTM by showing a matching $M\log T$ lower bound.
This is the most technically challenging result of the paper.
For a high-level overview of its proof, we refer the reader back to \Cref{s:tech-and-challenges}.
\begin{theorem}
    \label{t:lower-bound-full-M}
    There exists two numerical constants $c_1, c_2>0$ such that, for any $M\ge 2$ and any time horizon $T \ge c_2 M^4$, the worst-case regret of any algorithm satisfies
    \[
        \sup_{\nu\in\cD_M} R_T^\nu
    \ge
       c_1 M \log T \;,
    \]
    where $R_T^\nu$ is the regret at time $T$ of the algorithm when the underlying i.i.d.\ sequence of traders' valuations follows the distribution $\nu$, and $\cD_M$ is the set of all distributions $\nu$ with density bounded by $M$.
\end{theorem}

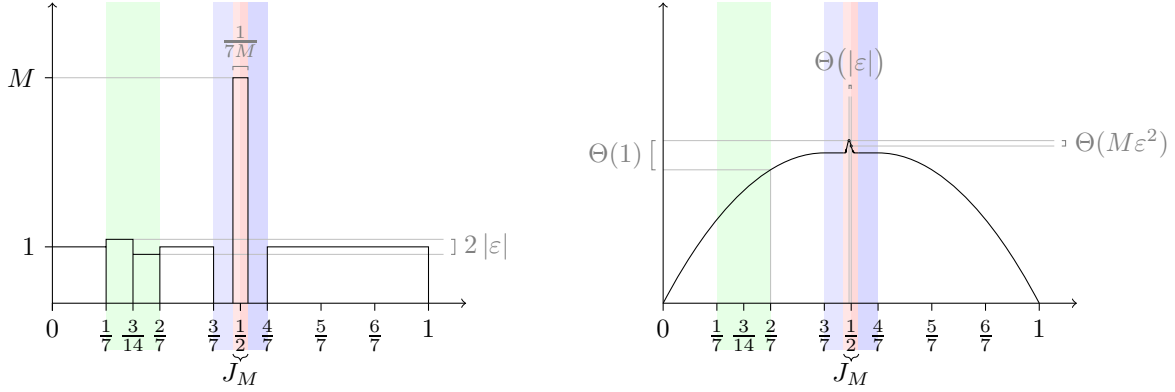
\begin{figure*}
    \centering
    \begin{tikzpicture}[scale=5]

    \fill[green!10!white] (1/7,-0.125) rectangle (2/7,0.8);
    \fill[blue!10!white] (3/7,-0.125) rectangle (0.48,0.8);
    \fill[red!10!white] (0.48,-0.125) rectangle (0.5,0.8);
    \fill[red!15!white] (0.5,-0.125) rectangle (0.52,0.8);
    \fill[blue!15!white] (0.52,-0.125) rectangle (4/7,0.8);
    
    \draw[gray] (0.48, 0.62) -- (0.48, 0.63) -- (0.52, 0.63) -- (0.52, 0.62)
        (1.06, 0.13) -- (1.07, 0.13) -- (1.07, 0.17) -- (1.06, 0.17)
        (1.07,0.15) node[right]{$2\labs{\e}$}
        (0.5,0.63) node[above]{$\frac1{7M}$}
    ;
    \draw[lightgray, very thin] (1.04,0.13) -- (2/7,0.13)
        (1.04,0.17) -- (3/14,0.17)
        (0,0.6) -- (0.48,0.6)
    ;
    
    \draw 
        (1/7,0) -- (1/7,0.17) -- (3/14,0.17) -- (3/14,0)
        (0,0.15) -- (1/7,0.15)
        (3/14,0.13) -- (2/7,0.13)
        (2/7,0) -- (2/7,0.15) -- (3/7,0.15) -- (3/7,0)
        (0.52, 0) -- (0.52,0.6) -- (0.48,0.6) -- (0.48,0)
        (0,-0.02) -- (0,0)
        (4/7,0) -- (4/7,0.15) -- (1, 0.15) -- (1,0)
        (0.5,-0.02) -- (0.5,0)
    ;
    
    \draw (0,0) -- (0,-0.02) node[below]{$0$}
        (1/7,0) -- (1/7,-0.02) node[below]{$\frac17$}
        (3/14,0) -- (3/14,-0.02) node[below]{$\frac{3}{14}$}
        (2/7,0) -- (2/7,-0.02) node[below]{$\frac27$}
        (3/7,0) -- (3/7,-0.02) node[below]{$\frac37$}
        (0.5,0) -- (0.5,-0.02) node[below]{$\frac12$}
        (4/7,0) -- (4/7,-0.02) node[below]{$\frac47$}
        (5/7,0) -- (5/7,-0.02) node[below]{$\frac57$}
        (6/7,0) -- (6/7,-0.02) node[below]{$\frac67$}
        (1,0) -- (1,-0.02) node[below]{$1$}
        (0,0.15) -- (-0.02,0.15) node[left]{$1$}
        (0,0.6) -- (-0.02,0.6) node[left]{$M$}
    ;

    \draw[<->] (0,0.8) -- (0,0) -- (1.1,0);

    \draw [decorate, decoration = {brace}] (0.52,-0.135) -- node[below]{$J_M$} (0.48,-0.135);
    \end{tikzpicture}
    \qquad
    \begin{tikzpicture}[scale=5]
    
    \fill[green!10!white] (1/7,-0.125) rectangle (2/7,0.8);
    \fill[blue!10!white] (3/7,-0.125) rectangle (0.48,0.8);
    \fill[red!10!white] (0.48,-0.125) rectangle (0.5,0.8);
    \fill[red!15!white] (0.5,-0.125) rectangle (0.52,0.8);
    \fill[blue!15!white] (0.52,-0.125) rectangle (4/7,0.8);

    \draw[gray] (1.06, 0.43269) -- (1.07, 0.43269) -- (1.07, 0.418311) -- (1.06, 0.418311)
        (1.07, {0.418311/2 + 0.43269/2})node[right]{$\Theta(M\e^2)$}
        (-0.02, 0.355) -- (-0.03, 0.355) -- (-0.03, 0.43269) -- (-0.02, 0.43269)
        (-0.03, {0.355/2 + 0.43269/2}) node[left]{$\Theta(1)$}
        (0.49429,0.57) -- (0.49429,0.58) -- (0.5,0.58) -- (0.5,0.57)
        ({0.49429/2 + 0.5/2}, 0.57) node[above]{$\Theta\brb{\labs{\e}}$}
    ;
    
    \draw[lightgray, very thin] 
        (1.04,0.418311) -- (0.5,0.418311)
        (1.04, 0.43269) -- (0.49429, 0.43269)
        (0,0.43269) -- (0.49429,0.43269)
        (0,0.355) -- (2/7, 0.355)
        (2/7,0) -- (2/7, 0.355)
        (0.5,0) -- (0.5,0.55)
        (0.49429,0) -- (0.49429,0.55)
    ;

    \draw[domain = 0.48:0.52, samples = 350] plot (\x, { 0.4 + 300*((\x-0.48)/0.04)^5 * ( 1 - (\x-0.48)/0.04)^9 });

    \draw (0.48, 0.4) -- (3/7, 0.4) parabola (0,0)
        (0.52, 0.4) -- (4/7, 0.4) parabola (1,0)
    ;
    
    \draw (0,0) -- (0,-0.02) node[below]{$0$}
        (1/7,0) -- (1/7,-0.02) node[below]{$\frac17$}
        (3/14,0) -- (3/14,-0.02) node[below]{$\frac{3}{14}$}
        (2/7,0) -- (2/7,-0.02) node[below]{$\frac27$}
        (3/7,0) -- (3/7,-0.02) node[below]{$\frac37$}
        (0.5,0) -- (0.5,-0.02) node[below]{$\frac12$}
        (4/7,0) -- (4/7,-0.02) node[below]{$\frac47$}
        (5/7,0) -- (5/7,-0.02) node[below]{$\frac57$}
        (6/7,0) -- (6/7,-0.02) node[below]{$\frac67$}
        (1,0) -- (1,-0.02) node[below]{$1$}
    ;

    \draw[very thin, fill=white] 
        (0.5,0.418311) circle (0.075pt)
        (0.49429, 0.43269) circle (0.075pt)
    ;

    \draw[<->] (0,0.8) -- (0,0) -- (1.1,0);

    \draw [decorate, decoration = {brace}] (0.52,-0.135) -- node[below]{$J_M$} (0.48,-0.135);
    \end{tikzpicture}
    \caption{On the left, the density $f_\e$ of a ``hard'' instance used to prove the lower bounds in \Cref{t:lower-bound-full-M,t:lower-bound-two-bit}. 
    A base uniform distribution is warped in the intervals $[1/7,2/7]$ (green) and $[3/7,4/7]$ (blue+red). 
    The density on $[1/7,2/7]$ is split into two uneven parts, differing by $\e$ from the original.
    The mass on $[3/7,4/7]$ is concentrated in a small set $J_M$ of size $\Theta(1/M)$ around $1/2$. 
    The corresponding gain from trade, on the right, has a smooth spike of height $\Theta(M\labs\e^2)$ situated in $J_M$, at a distance $\Theta(\labs\e)$ from $1/2$. 
    When $\e<0$ (resp., $\e>0$), the spike is left (resp., right) of $1/2$, and posting $1/2$ is better than posting any price after (resp., before) $1/2$. In the two-bit feedback lower bound, the only way to gather usable feedback is to post prices in $[1/7,2/7]$, which give rewards $\Theta(1)$-away from the optimal one.
    \label{f:lower-bound-density}}
\end{figure*}

\begin{proof}
    Given that we are in a stochastic i.i.d.\ setting, we can restrict this proof to deterministic algorithms without loss of generality.
    Let $M\ge 2$, 
    $J_M \coloneqq \bsb{\frac12 - \frac1{14M}, \frac12 + \frac1{14M}}$,
    $
        f
    \coloneqq 
        \I_{\lsb{0,\frac37}} 
        + M \I_{J_M}
        + \I_{\lsb{\frac47, 1}}
    $,
    and, for any $\e \in [-1,1]$, 
    $
        g_\e 
    \coloneqq 
        - \e \I_{\lsb{\frac17,\frac3{14}}} 
        + \e \I_{ \left( \frac3{14},\frac27 \right] }
    $
    and $f_\e \coloneqq f + g_\e$ (see \cref{f:lower-bound-density}, left).        For any $\e \in [-1,1]$, note that $0 \le f_\e \le M$ and $\int_0^1 f_{\e}(x) \dif x = 1$, hence $f_{\e}$ is a valid density on $[0,1]$ bounded by $M$, and we will denote the corresponding probability measure by $\nu_{\e}$.
    Consider for each $q \in [0,1]$, an i.i.d.\ sequence $(B_{q,t})_{t \in \N}$ of Bernoulli random variables of parameter $q$, an i.i.d.\ sequence $(\tilde{B}_t)_{t \in \N}$ of Bernoulli random variables of parameter $1/7$, an i.i.d. sequence $(U_t)_{t \in \N}$ of uniform random variables on $[0,1]$, a uniform random variable $E$ on $[-\e_M,\e_M]$ where $\e_M := \frac{7}{M}$, such that $\lrb{(B_{q,t})_{t \in \N, q \in [0,1]} , (\tilde{B}_t)_{t \in \N}, (U_t)_{t \in \N}, E}$ is an independent family. 
    Let $\varphi \colon [0,1] \to [0,1]$ be such that, if $U$ is a uniform random variable on $[0,1]$, then the distribution of $\varphi(U)$ has density $\frac{7}{6}\cdot f \cdot \I_{[0,1]\m [\fracc{1}{7},\fracc{2}{7}]}$ (which exists by the Skorokhod representation theorem \cite[Section 17.3]{williams1991probability}).
    For each $\e \in [-1,1]$ and $t \in \N$, define
    \begin{equation}
        V_{\e,t}
    \coloneqq
        \lrb{ \frac{2+U_t}{14}(1-B_{\frac{1+\e}{2},t}) + \frac{3+U_t}{14}B_{\frac{1+\e}{2},t}}\tilde{B}_t + \varphi(U_t)(1-\tilde{B}_t).
    \label{e:representation_of_V}
    \end{equation}
    Straightforward computations show that, for each $\e \in [-1,1]$ the sequence $(V_{\e,t})_{t \in \N}$ is i.i.d.\  with common distribution given by $\nu_\e$, and this sequence is independent of $E$.
    For any $\e\in[-1,1]$, $p\in[0,1]$, and $t\in \N$, let $\GFT_{\e,t}(p) \coloneqq \gft(p,V_{\e,2t-1}, V_{\e,2t})$ (for a qualitative representation of its expectation, see \cref{f:lower-bound-density}, right).
    For any $\e \in [-1,1]$ and $t \in \N$, a direct computation shows that $\bar\nu_\e = \E[V_{\e,t}] = \frac12 + \frac{\e}{196}$.
    By \Cref{l:key,l:repre}, we have, for all $\e \in [-1,1],t \in \N$, and $p \in [0,1]$,
    \[
        \E\bsb{ \GFT_{\e,t}(p) }
    =
        2\int_0^p\int_0^\lambda f_\e(s)\dif s \dif \lambda + 2 (\bar\nu_\e - p)\int_0^p f_\e(s) \dif s \;,
    \]
    which, together with the fundamental theorem of calculus ---\cite[Theorem 14.16]{bass2013real}, noting that $p\mapsto\E\bsb{ \GFT_{\e,t}(p) }$ is absolutely continuous with derivative defined a.e.\ by $p\mapsto 2(\bar\nu_\e-p) f_\e(p)$--- yields, for any $p\in J_M$,
    \begin{equation}
        \E\bsb{ \GFT_{\e,t}(\bar\nu_\e) } - \E\bsb{ \GFT_{\e,t}(p) }
    =
        M |\bar\nu_\e - p|^2 \;.
        \label{e:proof-lb}
    \end{equation}
    Note also that for all $\e \in [-\e_M,\e_M]$, $t\in \N$, and $p\in [0,1]\setminus J_M$, 
    \begin{equation}
        \E\bsb{ \GFT_{\e,t}(p) } 
    \le
        \E\lsb{ \GFT_{\e,t}\lrb{\fracc12} } \;.
    \label{e:reduction_to_J_M}
    \end{equation}
    Fix any arbitrary deterministic algorithm for the full feedback setting $(\tilde{\alpha}_t)_{t\in \N}$, i.e., a sequence of functions $\tilde{\alpha}_t \colon \brb{[0,1]\times[0,1]}^{t-1} \to [0,1]$ mapping past feedback into prices (with the convention that $\tilde{\alpha}_1$ is just a number in $[0,1]$). For each $t \in \N$, define $\alpha_t \colon \brb{[0,1]\times[0,1]}^{t-1} \to J_M$ equal to $\tilde{\alpha}_t$ whenever $\tilde{\alpha}_t$ takes values in $J_M$, and equal to $1/2$ otherwise. 
    Defining $Z\coloneqq \frac{1+E}{2}$, and $R_T^\nu$ as the regret of the algorithm $(\tilde{\alpha}_t)_{t \in \N}$ at time $T$ when the underlying sequence of traders' valuations follows the distribution $\nu$, we have that the worst-case regret $\sup_{\nu \in \cD_M} R_T^\nu$ is lower bounded by%
    {%
    \allowdisplaybreaks%
    \begin{align*}
    &
        \sup_{\e\in[-\e_M,\e_M]} 
        \sum_{t=1}^T \E\Bsb{ \GFT_{\e,t}(\bar\nu_\e) - \GFT_{\e,t}\brb{ \tilde{\alpha}_t( V_{\e,1}, \dots, V_{\e,2(t-1)} ) } }
    \\
    &
    \quad
    \overset{(\ref{e:reduction_to_J_M})}\ge
        \sup_{\e\in[-\e_M,\e_M]} 
        \sum_{t=1}^T \E\Bsb{ \GFT_{\e,t}(\bar\nu_\e) - \GFT_{\e,t}\brb{ \alpha_t( V_{\e,1}, \dots, V_{\e,2(t-1)} ) } }
    \\
    &
    \quad
    \overset{\spadesuit}=
        M \sup_{\e\in[-\e_M,\e_M]} 
        \sum_{t=1}^T \E\Bsb{ \babs{ \bar\nu_\e - \alpha_t(V_{\e,1}, \dots, V_{\e, 2(t-1)}) }^2 }
    \\
    &
    \quad
    \ge
        M \sum_{t=1}^T \E\Bsb{ \babs{ \bar\nu_E - \alpha_t(V_{E,1}, \dots, V_{E, 2(t-1)}) }^2 }
    \\
    &
    \quad
    \overset{\varheartsuit}\ge
        M \sum_{t=1}^T \E\Bsb{ \babs{ \bar\nu_E - \E[\bar\nu_E \mid V_{E,1}, \dots, V_{E,2(t-1)} ] }^2 }
    \\
    &
    \quad
    =
        \frac{M}{196} \sum_{t=1}^T \E\Bsb{ \babs{ E - \E[E \mid V_{E,1}, \dots, V_{E,2(t-1)} ] }^2 }
    \\
    &
    \quad
    \overset{\vardiamondsuit}\ge
        \frac{M}{196} \sum_{t=1}^T \E\Bsb{ \babs{ E - \E[E \mid B_{\frac{1+E}{2},1},\dots,B_{\frac{1+E}{2},2(t-1)}] }^2 }
    =
    \\
    &
    \quad
    =
        \frac{M}{98} \sum_{t=1}^T \E\Bsb{ \babs{ Z - \E[Z \mid B_{Z,1},\dots,B_{Z,2(t-1)}] }^2 }
    \end{align*}
    }%
    where $\spadesuit$ follows from \eqref{e:proof-lb} and the fact that $\alpha_t$ takes values in $J_M$;
    $\varheartsuit$ from the fact that the minimizer of the $L^2(\Pb)$-distance from $\bar\nu_E$ in $\sigma(V_{E,1}, \dots, V_{E,2(t-1)})$ is $\E[\bar\nu_E \mid V_{E,1}, \dots, V_{E,2(t-1)} ]$ (see, e.g., \cite[Section 9.4]{williams1991probability});
    $\vardiamondsuit$ follows from the fact that, by \cref{e:representation_of_V} and the independence of $E$ from $\lrb{(B_{q,t})_{t \in \N, q \in [0,1]} , (\tilde{B}_t)_{t \in \N}, (U_t)_{t \in \N}}$, the conditional expectation $\E[E \mid V_{E,1}, \dots, V_{E,2(t-1)} ]$ is a measurable function of $B_{\frac{1+E}{2},1},\dots,B_{\frac{1+E}{2},2(t-1)}$, together with the same observation made in $\varheartsuit$ about the minimization of $L^2(\Pb)$ distance.

    Finally, the general term of this last sum is the expected squared distance between the random parameter (drawn uniformly over $[ \fracc{(1 - \e_M)}{2}, \fracc{(1 + \e_M)}{2} ]$) of an i.i.d.\ sequence of Bernoulli random variables and the conditional expectation of this random parameter given $2(t-1)$ independent realizations of these Bernoullis. 
    A probabilistic argument shows that there exist two universal constants $\tilde{c}, c_2>0$ such that, for all $t \ge c_2 M^4$, 
    \begin{equation}
        \label{e:probabilistic}
        \E\Bsb{ \babs{ Z - \E[Z \mid B_{Z,1}, \dots, B_{Z,2(t-1)} ] }^2 }
    \ge
        \tilde c\frac{1}{t-1} \;.
    \end{equation}
    At a high level, this is because, in an event of probability $\Omega(1)$, if $t$ is large enough, the conditional expectation $\E[Z \mid B_{Z,1}, \dots, B_{Z,2(t-1)} ]$ is very close to the empirical average $\frac{1}{2(t-1)}\sum_{s=1}^{2(t-1)} B_{Z,s}$, whose expected squared distance from $Z$ is $\Omega\brb{ 1/(t-1)}$. 
    For a formal proof \eqref{e:probabilistic} with explicit constants, see the supplementary material.
    Summing over $t$ and putting everything together gives the result.
\end{proof}

\section{Two-Bit Feedback}
We now study the two-bit feedback case (corresponding to posted price mechanisms) under the bounded density assumption.

\subsection{Upper bound}
Motivated once more by the intuition provided by \Cref{t:mysticus}, we begin this section by giving a way to approximate the expected value of traders' valuations on the basis of the two-bit feedback and quantify the approximation power of this strategy.
\begin{lemma}
\label{l:approx}
    For any random variable $X$ on $[0,1]$ and any $T_0 \in \N$,
\[
    0
\le
    \E[X] - \frac{1}{T_0} \sum_{t=1}^{T_0} \Pb \lsb{ \frac{t}{T_0} \le X }
\le
    \frac{1}{T_0}
\]
\end{lemma}
\begin{proof}
Notice that
\begin{align*}
&
    \frac{1}{T_0} \sum_{t=1}^{T_0} \Pb\lsb{  \frac{t}{T_0} \le X }
=
    \sum_{t=1}^{T_0} \int_{\frac{t-1}{T_0}}^{\frac{t}{T_0}}  \Pb\lsb{ \frac{t}{T_0} \le X  } \dif \lambda
\\
&
\qquad
\le
    \sum_{t=1}^{T_0} \int_{\frac{t-1}{T_0}}^{\frac{t}{T_0}}  \Pb\lsb{ \lambda \le X } \dif \lambda
\\
&
\qquad
\le
    \sum_{t=1}^{T_0} \int_{\frac{t-1}{T_0}}^{\frac{t}{T_0}}  \Pb\lsb{ \frac{t-1}{T_0} \le X  }
=
    \frac{1}{T_0} \sum_{t=1}^{T_0}  \Pb\lsb{ \frac{t-1}{T_0} \le X  } \dif \lambda\;.
\end{align*}
Since by Fubini's Theorem,
\[
    \E[X]
=
    \int_0^1\Pb[\lambda \le X] \dif \lambda
=
    \sum_{t=1}^{T_0} \int_{\frac{t-1}{T_0}}^{\frac{t}{T_0}} \Pb[\lambda\le X] \dif \lambda \;,
\]
we obtain
\begin{align*}
    0
&
\le
    T_0 \E[X] - \sum_{t=1}^{T_0}  \Pb\lsb{ \frac{t}{T_0} \le X  }
\le
    \sum_{t=1}^{T_0} \lrb{ \Pb\lsb{ \frac{t-1}{T_0} \le X  } - \Pb\lsb{ \frac{t}{T_0} \le X  } }
\\
&
=
    \sum_{t=1}^{T_0} \Pb\lsb{ \frac{t-1}{T_0} \le X < \frac{t}{T_0} }
=
    \Pb[0\le X<1]
\le
    1 \;. \qedhere
\end{align*}
\end{proof}
The previous lemma suggests the design of a simple Explore-then-Commit (ETC) strategy (\Cref{a:etc}), where the learner spends an initial phase of length $T_0$ trying to estimate $\E[V]$ and then posts this estimate every round up to the time horizon $T$.
\begin{algorithm}
\textbf{Input:} Exploration time $T_0 \in \N$\;
\For
{%
    time $t=1,2,\ldots, T_0$
}
{
    Post $P_t \coloneqq t/T_0$\;
    Receive feedback $\I\{P_t \le V_{2t-1}\}$, $\I\{P_t \le V_{2t}\}$\;
}
\For
{%
    time $t=T_0+1,T_0+2,\ldots$
}
{
    Post $P_t \coloneqq \frac{1}{2T_0} \sum_{s=1}^{T_0} \brb{ \I \{P_s \le V_{2s-1} \} + \I \{P_s \le V_{2s}\}}$\;
}
\caption{Explore-then-Commit (ETC)}
\label{a:etc}
\end{algorithm}
ETC algorithms are of great practical importance due to their easy implementability and interpretability.
This usually comes at a cost of performance. 
As the following result (together with \Cref{t:lower-bound-two-bit}, in the next section) will show, our ECT algorithm is free of this flaw. 
\begin{theorem}
    \label{t:ETC}
    If $\nu$ has density bounded by some constant $M>0$, then the regret of ETC satisfies, for all time horizons $T$,
\[
    R_T
\le
    T_0 - \frac12 + M (T-T_0) \lrb{ \frac{2}{T_0^2} + \frac{1}{T_0} }
\]
Tuning the parameter $T_0\coloneqq\bce{\sqrt{MT}}$ yields
\[
    R_T
\le
    2.5 + 2\sqrt{MT} \;.
\]    
\end{theorem}
\begin{proof}
Fix any $T_0\in\N$ and let $p_0 \coloneqq \frac{1}{T_0}\sum_{s=1}^{T_0} \Pb\lsb{ \frac{s}{T_0} \le V }$.
By Hoeffding's inequality and Fubini's theorem, we get
\begin{align*}
    \E \lsb{ \labs{ p_0  - P_{T_0+1} }^2 }
&
=
    \int_0^{+\iop} \Pb \lsb{ \labs{ p_0  - P_{T_0+1} }^2 \ge \e } \dif \e
\\
&
\le
    \int_0^{+\iop} 2 \exp(-4 \e T_0 ) \dif \e
=
    \frac{1}{2T_0} \;,
\end{align*}
from which, leveraging also \cref{l:approx}, it follows that
\begin{align*}
    \E\Bsb{ \babs{ \E[V] - P_{T_0+1} }^2 }
&
\le
    2 \labs{ \E[V] - p_0  }^2 + 2 \E\Bsb{ \labs{  p_0 - P_{T_0+1} }^2 }
\le
    \frac{2}{T_0^2} + \frac{1}{T_0}\;.
\end{align*}
Proceeding as in the proof of \Cref{t:ftm}, we obtain, for all $t\in \N$,
\[
    \E\Bsb{ \GFT_t\brb{ \E[V] } - \GFT_t(P_t) }
\le
    M \E\Bsb{ \babs{ \E[V] - P_t }^2 } \;.
\]
Putting everything together, we get, for all $T \ge T_0+1$
\begin{align*}
&
    R_T - T_0 + \frac12
\le
    \sum_{t=T_0+1}^T \E \Bsb{ \GFT_t\brb{ \E[V] } - \GFT_t(P_t) }
\\
&
\qquad
\le
    M \sum_{t=T_0+1}^T \E \Bsb{ \babs{ \E[V] - P_t }^2 }
=
    M \sum_{t=T_0+1}^T \E \Bsb{ \babs{ \E[V] - P_{T_0+1} }^2 }
\\
&
\qquad
\le
    M (T-T_0) \lrb{ \frac{2}{T_0^2} + \frac{1}{T_0} } \;.
\end{align*}

Substituting the selected parameters in the final expression yields the second part of the result.
\end{proof}

\subsection{Lower bound}
In this section, we prove the optimality of our ETC algorithm by showing a matching $\sqrt{M T}$ lower bound.
For a high-level overview of its proof, we refer the reader back to \Cref{s:tech-and-challenges}.

\begin{theorem}
    \label{t:lower-bound-two-bit}
    There exists two numerical constants $c_1, c_2>0$ such that, for any $M\ge 2$ and any time horizon $T \ge c_2 M^3$, the worst-case regret of any algorithm satisfies
    \[
        \sup_{\nu} R_T^\nu
    \ge
       c_1 \sqrt{ M T } \;,
    \]
    where $R_T^\nu$ is the regret at time $T$ of the algorithm when the underlying i.i.d.\ sequence of traders' valuations follows the distribution $\nu$, and $\cD_M$ is the set of all distributions $\nu$ with density bounded by $M$.
\end{theorem}
\begin{proof}[Proof sketch.]
Fix $M\ge 2$ and $T\in \N$.
We will use the same random variables, distributions, densities, and notation as in the proof of \Cref{t:lower-bound-full-M}.
We will show that for each algorithm for the 2-bit feedback setting and each time horizon $T$, if $R_T^{\nu}$ is the regret of the algorithm at time horizon $T$ when the underlying distribution of the traders' valuations is $\nu$, then $\max\brb{ R_T^{\nu_{-\e}}, R_T^{\nu_{+\e}} } = \Omega\brb{ \sqrt{MT} }$ if $T=\Omega(M^3)$.

Note that for all $\e>0$, $t \in \N$, and $p<\frac12$
\begin{equation}
    \E\lsb{ \GFT_{\e,t} \lrb{\fracc12} }
\ge
    \E\bsb{ \GFT_{\e,t} (p) } \;.
    \label{e:lower-bound-1}
\end{equation}
Similarly, for all $\e<0$, $t \in \N$, and $p>\frac12$,
\begin{equation}
    \E\lsb{ \GFT_{\e,t} \lrb{\fracc12} } \ge \E\bsb{ \GFT_{\e,t} (p) } \;.
    \label{e:lower-bound-2}
\end{equation}
Furthermore, a direct verification shows that, for each $\e \in [-1,1]$ and $t \in \N$, 
\begin{equation}
    \max_{p \in [0,1]} \E\bsb{ \GFT_{\e,t} (p)} - \max_{p \in [\frac{1}{7},\frac{2}{7}]} \E\bsb{ \GFT_{\e,t} (p)} \ge \frac{1}{50} = \Omega(1) \;.
    \label{e:lower-bound-3}
\end{equation}
Now, assume that $T \ge M^3/14^4$ so that, defining  $\e \coloneqq (MT)^{-1/4}$, we have that the maximizer of the expected gain from trade $\frac12 +\frac\e{196}$ belongs to the spike region $J_M$.
In the $+\e$ (resp., $-\e$) case, the optimal price belongs to the region $\bigl(\frac{1}{2},\frac{1}{2}+\frac{1}{14M}\bigr]$ (resp., $\bigl[\frac12 - \frac1{14M}, \frac12\bigr)$).
By posting prices in the wrong region $\bsb{0, \frac{1}{2}}$ (resp., $\bsb{\frac{1}{2}, 1}$) in the $+\e$ (resp., $-\e$) case, the learner incurs a $\Omega(M\e^2) = \Omega\brb{\sqrt{M/T}}$ instantaneous regret by \eqref{e:proof-lb} and \eqref{e:lower-bound-1} (resp., \eqref{e:proof-lb} and \eqref{e:lower-bound-2}).
Then, in order to attempt suffering less than $\Omega\brb{\sqrt{M/T} \cdot T} = \Omega\brb{\sqrt{MT }}$ regret, the algorithm would have to detect the sign of $\pm\e$ and play accordingly.
We will show now that even this strategy will not improve the regret of the algorithm (by more than a constant) because of the cost of determining the sign of $\pm\e$ with the available feedback. 
Since the feedback received from the two traders at time $t$ by posting a price $p$ is $\I\{p \le V_{\pm\e,2t-1}\}$ and $\I\{ p \le V_{\pm\e,2t}\}$, the only way to obtain information about (the sign of) $\pm\e$ is to post in the costly ($\Omega(1)$-instantaneous regret by \cref{e:lower-bound-3}) sub-optimal region $[\frac{1}{7},\frac{2}{7}]$ (see \cref{f:lower-bound-density}).
However, posting prices in the region $[\frac{1}{7},\frac{2}{7}]$ at time $t$ can't give more information about $\pm\e$ than the information carried by $V_{\pm\e,2t-1}$ and $V_{\pm\e,2t}$, which, in turn, can't give more information about $\pm\e$ than the information carried by the two Bernoullis $B_{\frac{1\pm\e}{2},2t-1}$ and $B_{\frac{1\pm\e}{2},2t}$. 
Since (via an information-theoretic argument) in order to distinguish the sign of $\pm\e$ having access to i.i.d.\ Bernoulli random variables of parameter $\frac{1\pm\e}{2}$ requires $\Omega(1/\e^2) = \Omega\brb{ \sqrt{MT} }$ samples, we are forced to post at least $\Omega\brb{ \sqrt{MT} }$ prices in the costly region $\bsb{\frac{1}{7},\frac{2}{7}}$ suffering a regret of $\Omega\brb{ \sqrt{MT} } \cdot \Omega(1) = \Omega\brb{ \sqrt{MT}}$.
Putting everything together, no matter what the strategy, each algorithm will pay at least $\Omega\brb{ \sqrt{MT}}$ regret.
\end{proof}

\section{Beyond bounded densities}
In this section, we investigate how the problem changes when the bounded density assumption is no longer guaranteed to hold.

\subsection{The full feedback case}
Our Representation Lemma shows that our goal is to maximize $\rho(\nu)$, where $\rho$ is known (by \Cref{l:repre}) but $\nu$ is not (by assumption). 
A natural strategy is then to approximate $\nu$ through an empirical distribution $\hat\nu_t$ and then maximize $\rho(\hat\nu_t)$.
\begin{algorithm}
Post $P_1 \coloneqq 1/2$, then receive feedback $V_{1}$, $V_{2}$\;
\For
{%
    time $t=2,3,\dots$
}
{
    Let $\hat\nu_t \coloneqq \frac{1}{2(t-1)}\sum_{s=1}^{2(t-1)} \delta_{V_s}$\;
    Post $P_t \in \argmax_{p\in[0,1]}\coloneqq \rho (\hat \nu_t) (p) $\;
    Receive feedback $V_{2t-1}$, $V_{2t}$\;
}
\caption{Follow-the-$\rho$ (\ftr)}
\label{a:ftr}
\end{algorithm}
This is precisely what we do in \Cref{a:ftr}.
Note that maximizing $\rho(\hat\nu_t)$ can be done efficiently by an exhaustive search of $\Theta(t)$ candidate points.
The next result gives a regret guarantee of $\sqrt{T}$ for \ftr{}.
\begin{theorem}
\label{t:ftr}
The regret of \ftr{} satisfies, for all time horizons $T$,
\[
    R_T
\le
     \fracc12 + 4 \brb{ 3\sqrt{\pi} + \sqrt{2} } \sqrt{T-1} \;.
\]
\end{theorem}
\begin{proof}[Proof sketch.]
For any price $p^\star\in[0,1]$, and time step $t\ge2$,
{%
\allowdisplaybreaks%
\begin{align*}
&
    \E\bsb{ \GFT_t(p^\star) } - \E\bsb{ \GFT_t(P_t) }
\overset{\spadesuit}{=}
    \rho(\nu)(p^\star) - \E\bsb{ \rho(\nu)(P_t) }
\\
&
\quad
=
    \E\bsb{ \rho(\nu)(p^\star) - \rho(\hat\nu_t)(p^\star) }
    +
    \E\bsb{ \rho(\hat\nu_t)(p^\star) - \rho(\hat\nu_t)(P_t) }
\\
&
\qquad
    +
    \E\bsb{ \rho(\hat\nu_t)(P_t) - \rho(\nu)(P_t) }
\overset{\clubsuit}{\le}
    2\E \lsb{ \sup_{p\in[0,1]} \babs{ \rho(\hat\nu_t)(p)-\rho(\nu)(p) } }
\\
&
\quad
\overset{\vardiamondsuit}{\le}
    2\E\lsb{6 \sup_{\lambda \in [0,1] }\labs{(\nu - \nu_t)[0,\lambda]} } +2 \E\lsb{ \labs{ \E[V] - \frac{\sum_{s=1}^{2(t-1)} V_s}{2(t-1)} } }
\\
&
\quad
\overset{\varheartsuit}{\le}
    12 \int_0^{+\infty}\Pb\lsb{ \sup_{\lambda \in [0,1]} \labs{ (\nu-\nu_t)[0,\lambda] } \ge \e }\dif \e + \frac{4}{\sqrt{2(t-1)}}  
\\
&
\quad
\overset{\bullet}{\le}
    24 \int_0^{+\infty} e^{4\e^2(t-1)}\dif \e + \frac{4}{\sqrt{2(t-1)}}  
\le
    2  \frac{3\sqrt{\pi} + \sqrt{2} }{\sqrt{t-1}} \;,
\end{align*}%
}%
where $\spadesuit$ follows by \cref{l:repre} and the Freezing Lemma; 
$\clubsuit$ by definition of $P_t$;
$\vardiamondsuit$ by elementary computations;
$\varheartsuit$ by Fubini's Theorem and upper bounding $\E\bsb{ \labs{\dots} }$ with the variance of $\frac{\sum_{s=1}^{2(t-1)} V_s}{2(t-1)}$;
and $\bullet$ by the DKW inequality (see, e.g., \cite[Theorem J.1]{cesa2023bilateral}).
Summing over $t$ from $2$ to $T$, yields the result.  
\end{proof}
Next, we show that running FTM until evidence is observed that $\nu$ does not have a bounded density (i.e., until we observe the same sample twice), then switching to \ftr{} (\ftmtr{}, \Cref{a:ftmtr}),
\begin{algorithm}
\For
{%
    time $t=1,2,\dots$
}
{
    Post $P_t$ according to FTM (\cref{a:ftm})\;
    \If
    {%
        $\{V_{2t-1},V_t\} \cap \{V_{1},\dots, V_{2(t-1)}\} \neq \varnothing$
    }
    {%
        $\tau \coloneqq t$ and \textbf{break}\;
    }
}
Run FT$\rho$ (\cref{a:ftr}) up to time $\tau$ without posting prices\;
\For
{%
    time $t=\tau+1,\tau+2,\dots$
}
{%
    Post $P_t$ according to FT$\rho$ (\cref{a:ftr})\;
}
\caption{Follow-the-Mean-then-$\rho$ (\ftmtr)}
\label{a:ftmtr}
\end{algorithm}
keeps the guarantees of FTM if $\nu$ has a bounded density and those of \ftr{} if $\nu$ does not, without requiring this \emph{a priori} knowledge. 
\begin{theorem}
\label{t:ftmtr}
For all time horizons $T$, the regret of \ftmtr{} satisfies
\begin{enumerate}
    \item \label{i:first} If $\nu$ has density bounded by some constant $M>0$ and $T\ge2$,
    \[
        R_T
    \le
        \frac12 + \frac{M}{4}\brb{1 + \ln (T-1)} \;.
    \]
    \item \label{i:second} Otherwise,
    $
        R_T
    \le
        7.5 + 6 \brb{ 2\sqrt{\pi} + \sqrt{2} } \sqrt{T-1}\;.
    $
\end{enumerate}
\end{theorem}
\begin{proof}
    If $\nu$ has density bounded by some constant $M>0$, then the condition $\{V_{2t-1}, V_t\} \cap \{V_{1},\dots, V_{2(t-1)}\} \neq \varnothing$ never occurs with probability $1$; therefore, the expected regret of \ftmtr{} coincides with that of FTM, and \Cref{i:first} follows by \Cref{t:ftm}.
    
    For \Cref{i:second}, define the following:
    $\e \coloneqq \sup_{p\in[0,1]}\nu\{p\}$; 
    $p_\e\in[0,1]$ such that $\nu\{p\} \ge \e/2$;
    $\tau' \coloneqq \tau\wedge T$;
    $\tau'' \coloneqq \inf \bcb{ t\in \N \mid \sum_{s=1}^t \I \{ V_{2s} = p_\e \} \ge 2}$, and note that $\tau' \le \tau \le \tau''$;
    $P_1',P_2',\dots$ (resp., $P_1'',P_2'',\dots$) are the prices posted by FTM (resp., \ftr) run with feedback $V_1,V_2, \dots$;
    $p^\star \in \argmax_{p\in [0,1]} \rho(\nu)(p)$;
    $X_t(p) \coloneqq \GFT_t(p^\star) - \GFT_t(p)$ for all $t\in \N$ and $p\in[0,1]$; $x(p) \coloneqq \rho(\nu)(p^\star) - \rho(\nu)(p)$ for all $p\in[0,1]$, and note that $x\ge0$.
    Then:
    \begin{align*}
        R_T
    &
    =
        \E\lsb{ \sum_{t=1}^T x(P_t) }
    =
        \E\lsb{ \sum_{t=1}^{\tau'} x(P_t') } + \E\lsb{ \sum_{t=\tau'+1}^T x(P_t'') }
    \\
    &
    \le \E\lsb{ \sum_{t=1}^{\tau'} x(P_t') } + \E\lsb{ \sum_{t=1}^T x(P_t'') },
    \end{align*}
    where the first equality follows from the Freezing Lemma, the inequality from $x\ge 0$, and note that the second expectation in the last formula is the regret of \ftr{} (by the Freezing Lemma),  and can be controlled applying \Cref{t:ftr}. 
    For the first term, we have
    \begin{align*}
    &
        \E\lsb{ \sum_{t=1}^{\tau'} x(P_t') } 
        -
        \e\E[\tau']
    \le
        \E\lsb{ \sum_{t=1}^{\tau'} \rhot(\nu)\brb{ \E[V] } - \sum_{t=1}^{\tau'} \rhot(\nu)(P_t') }
    \\
    &
    \qquad
    \le
        2  \sum_{t=1}^{T} \E\Bsb{ \babs{ \E[V] - P_t' } }
    \le
        1  + 2 \sum_{t=2}^{T} \sqrt{ \var \lrb{ \frac{\sum_{s=1}^{2(t-1)} V_s}{2(t-1)} } }
    \\
    &
    \qquad
    =
        1  + \sqrt 2 \sum_{t=1}^{T-1} t^{-1/2}
    \le
        1  + 2 \sqrt{ 2 (T-1) }\;,
    \end{align*}
    where the first inequality follows from the definition of $\rho$ and the second  from \Cref{l:key}.
    Noting that $\tau''$ has a negative binomial distribution with parameters $2$ and $\e/2$, we get
    \[
        \e\E[\tau']
    \le
        \e\E[\tau'']
    =
        \e \cdot 2 \frac{1-\e/2}{\e/2}
    \le
        4 \;.
    \]
    Putting everything together gives the result.
\end{proof}
We now show that when $\nu$ does not have a bounded density, the $\sqrt{T}$ guarantee of \ftr{} and \ftmtr{} is optimal.
\begin{theorem}
    \label{t:lower-bound-full-general}
    There exists a numerical constant $c > 0$ such that, for any time horizon $T$, the worst-case regret of any algorithm satisfies
    \[
        \sup_{\nu} R_T
    \ge
       c \sqrt{T} \;,
    \]
    where the $\sup$ is over all distributions $\nu$.
\end{theorem}

\begin{proof}[Proof sketch.]
    For each $\e \in \bsb{ -\frac{1}{8},\frac{1}{8} }$, consider the distribution
    $
        \nu_\e \coloneqq \frac{1}{4}\delta_0 + \brb{\frac{1}{4} + \e}\delta_{1/3} + \brb{\frac{1}{4} - \e}\delta_{2/3} + \frac{1}{4}\delta_{1}
    $.
    Then, prices $p\in\bcb{\frac13,\frac23}$ are either $\Theta(\e)$-suboptimal or optimal (depending on the sign of $\e$), and the remaining prices $p$ in $[0,1]$ are $\Omega(1)$-suboptimal.    
    Noting that $\Omega(\e T)$ is the regret paid by posting suboptimal prices for $T$ time steps and, via an information-theoretic argument, that $\Omega(1/\e^2)$ rounds are needed to determine the sign of $\e$, one can conclude that the regret of any algorithm is $\Omega\brb{ \min\brb{ \e T, \e \frac{1}{\e^2}} } = \Omega\brb{\sqrt{T}}$, if $\e = T^{-1/2}$.
    See the supplementary material for additional details. 
\end{proof}

\subsection{The two-bit feedback case}
Finally, learning becomes impossible if the bounded density assumption on $\nu$ is lifted in the two-bit feedback case. We show this by leveraging a needle-in-a-haystack phenomenon. The idea is that any deterministic algorithm is capable of posting just a finite number of prices since all the possible sequences of 2-bit feedback, up to a certain time horizon $T$, are $4^T$. The goal is then to design a hard instance whose maximizer is never played by the algorithm and whose value is $\Omega(1)$ higher than the other values. 
\begin{theorem}
    \label{t:lower-bound-two-bit-general}
    For any time horizon $T$, the worst-case regret of any algorithm satisfies
    \[
        \sup_{\nu} R_T
    \ge
       \frac{T}{9} \;,
    \]
    where the $\sup$ is over all distributions $\nu$.
\end{theorem}

\begin{proof}
    Given that we are in a stochastic i.i.d.\ setting, it is enough to consider deterministic algorithms.
    Let $(\alpha_t)_{t \in \N}$ be a deterministic algorithm for the two-bit feedback setting, i.e., a sequence of functions $\alpha_t \colon \brb{\{0,1\}\times\{0,1\}}^{t-1} \to [0,1]$ mapping past feedback into prices (with the convention that $\alpha_1$ is just a number in $[0,1]$).
    Fix a time horizon $T$.
    Note that there are $4^T$ different sequences of pairs of zeroes and ones representing the feedback the algorithm could receive up to time $T$. Hence, the algorithm $(\alpha_t)_{t \in \N}$ selects the prices it posts in a set $A$ which has no more than $4^T$ different prices.
    For each $\eta \in \lrb{0,\frac{1}{2}}$, select any $x \in \lrb{\frac{1}{2}-\eta, \frac{1}{2}} \m A$ and define
    \[
        \nu_x \coloneqq \frac{1}{3}\delta_0 + \frac{1}{3}\delta_{x} + \frac{1}{3} \delta_{1} \;.
    \]
    Consider an i.i.d.\ sequence of traders' valuations $V_1,V_2,\dots$ with common distribution $\nu_x$, and for each $t \in \N$ and $p\in[0,1]$, let $\GFT_{t}(p) \coloneqq \gft(p,V_{2t-1}, V_{2t})$.
    For each $t \in \N$, notice that
    \begin{align*}
    &
        \E \bsb{ \GFT_{t}(x) }
    =
        \max_{p \in [0,1]} \E \bsb{ \GFT_{t}(p) }
    \\
    &
        \E \bsb{ \GFT_{t}(x) } - \max_{p \in [0,1]\m\{x\}} \E \bsb{ \GFT_{t}(p) }
    =
        \frac{4}{9} - \frac{4}{9} + \frac{2}{9}x
    \ge
        \frac{1-2\eta}{9} \;.
    \end{align*}
    Given that the algorithm never posts the price $x$ up to time $T$, it follows that the regret at time $T$ of the algorithm is lower bounded by $\frac{1-2\eta}{9} T$.
    Given that $\eta$ was arbitrarily chosen, the worst-case regret of the algorithm at time $T$ is lower bounded by $\frac{T}{9}$.
\end{proof}

\section{Conclusions and Open Problems}

Motivated by trading in OTC markets, we investigated the online learning problem of brokerage between agents with flexible seller/buyer roles, for which we provided a complete picture with tight upper and lower bounds in all 
the proposed settings.

Our paper motivates further study of brokerage in a few different directions. For example, it would be interesting to study the contextual version of this problem, where a context related to the traders' valuations is available to the broker \emph{before} making a decision, and, possibly, multiple traders arrive at each time step.
Finally, it would be intriguing to discover intermediate cases for which something meaningful could be said between the stationary (that we fully fleshed out) and adversarial (that, as discussed in the Related Work section, is unlearnable) cases.

\section*{Acknowledgments}
NB and TC acknowledges the support of NSERC through USRA (Undergraduate Student Research Awards).
TC also acknowledges the support of NSERC
through grants DGECR (Discovery Grant) and RGPIN (Discovery Grant Supplement).
RC acknowledges the financial support from the following sources: MUR PRIN grant 2022EKNE5K (Learning in Markets and Society), funded by the NextGenerationEU program within the PNRR scheme (M4C2, investment 1.1);
FAIR (Future Artificial Intelligence Research) project, funded by the NextGenerationEU program within the PNRR-PE-AI scheme (M4C2, investment 1.3, line on Artificial Intelligence); the EU Horizon CL4-2022-HUMAN-02 research and innovation action under grant agreement 101120237, project ELIAS (European Lighthouse of AI for Sustainability).


\bibliographystyle{ACM-Reference-Format} 
\bibliography{biblio}


\begin{thebibliography}{23}


\ifx \showCODEN    \undefined \def \showCODEN     #1{\unskip}     \fi
\ifx \showDOI      \undefined \def \showDOI       #1{#1}\fi
\ifx \showISBNx    \undefined \def \showISBNx     #1{\unskip}     \fi
\ifx \showISBNxiii \undefined \def \showISBNxiii  #1{\unskip}     \fi
\ifx \showISSN     \undefined \def \showISSN      #1{\unskip}     \fi
\ifx \showLCCN     \undefined \def \showLCCN      #1{\unskip}     \fi
\ifx \shownote     \undefined \def \shownote      #1{#1}          \fi
\ifx \showarticletitle \undefined \def \showarticletitle #1{#1}   \fi
\ifx \showURL      \undefined \def \showURL       {\relax}        \fi
\providecommand\bibfield[2]{#2}
\providecommand\bibinfo[2]{#2}
\providecommand\natexlab[1]{#1}
\providecommand\showeprint[2][]{arXiv:#2}

\bibitem[\protect\citeauthoryear{Archbold, de~Keijzer, and Ventre}{Archbold
  et~al\mbox{.}}{2023}]%
        {archbold2023non}
\bibfield{author}{\bibinfo{person}{Thomas Archbold}, \bibinfo{person}{Bart de
  Keijzer}, {and} \bibinfo{person}{Carmine Ventre}.}
  \bibinfo{year}{2023}\natexlab{}.
\newblock \showarticletitle{Non-Obvious Manipulability for Single-Parameter
  Agents and Bilateral Trade}. In \bibinfo{booktitle}{\emph{Proceedings of the
  2023 International Conference on Autonomous Agents and Multiagent Systems}}.
  \bibinfo{publisher}{International Foundation for Autonomous Agents and
  Multiagent Systems}, \bibinfo{address}{USA}, \bibinfo{pages}{2107--2115}.
\newblock


\bibitem[\protect\citeauthoryear{Azar, Fiat, and Fusco}{Azar
  et~al\mbox{.}}{2022}]%
        {azar2022alpha}
\bibfield{author}{\bibinfo{person}{Yossi Azar}, \bibinfo{person}{Amos Fiat},
  {and} \bibinfo{person}{Federico Fusco}.} \bibinfo{year}{2022}\natexlab{}.
\newblock \showarticletitle{An alpha-regret analysis of Adversarial Bilateral
  Trade}.
\newblock \bibinfo{journal}{\emph{Advances in Neural Information Processing
  Systems}}  \bibinfo{volume}{35} (\bibinfo{year}{2022}),
  \bibinfo{pages}{1685--1697}.
\newblock


\bibitem[\protect\citeauthoryear{Babaioff, Goldner, and Gonczarowski}{Babaioff
  et~al\mbox{.}}{2020}]%
        {babaioff2020bulow}
\bibfield{author}{\bibinfo{person}{Moshe Babaioff}, \bibinfo{person}{Kira
  Goldner}, {and} \bibinfo{person}{Yannai~A. Gonczarowski}.}
  \bibinfo{year}{2020}\natexlab{}.
\newblock \showarticletitle{Bulow-Klemperer-Style Results for Welfare
  Maximization in Two-Sided Markets}. In \bibinfo{booktitle}{\emph{Proceedings
  of the Thirty-First Annual ACM-SIAM Symposium on Discrete Algorithms}} (Salt
  Lake City, Utah) \emph{(\bibinfo{series}{SODA '20})}.
  \bibinfo{publisher}{Society for Industrial and Applied Mathematics},
  \bibinfo{address}{USA}, \bibinfo{pages}{2452–2471}.
\newblock


\bibitem[\protect\citeauthoryear{Bass}{Bass}{2013}]%
        {bass2013real}
\bibfield{author}{\bibinfo{person}{Richard~F Bass}.}
  \bibinfo{year}{2013}\natexlab{}.
\newblock \bibinfo{booktitle}{\emph{Real analysis for graduate students}}.
\newblock \bibinfo{publisher}{Createspace Ind Pub}, \bibinfo{address}{USA}.
\newblock


\bibitem[\protect\citeauthoryear{Blumrosen and Mizrahi}{Blumrosen and
  Mizrahi}{2016}]%
        {BlumrosenM16}
\bibfield{author}{\bibinfo{person}{Liad Blumrosen} {and}
  \bibinfo{person}{Yehonatan Mizrahi}.} \bibinfo{year}{2016}\natexlab{}.
\newblock \showarticletitle{Approximating Gains-from-Trade in Bilateral
  Trading}. In \bibinfo{booktitle}{\emph{Web and Internet Economics,
  {WINE}'16}} \emph{(\bibinfo{series}{Lecture Notes in Computer Science},
  Vol.~\bibinfo{volume}{10123})}. \bibinfo{publisher}{Springer},
  \bibinfo{address}{Germany}, \bibinfo{pages}{400--413}.
\newblock


\bibitem[\protect\citeauthoryear{Brustle, Cai, Wu, and Zhao}{Brustle
  et~al\mbox{.}}{2017}]%
        {brustle2017approximating}
\bibfield{author}{\bibinfo{person}{Johannes Brustle}, \bibinfo{person}{Yang
  Cai}, \bibinfo{person}{Fa Wu}, {and} \bibinfo{person}{Mingfei Zhao}.}
  \bibinfo{year}{2017}\natexlab{}.
\newblock \showarticletitle{Approximating Gains from Trade in Two-Sided Markets
  via Simple Mechanisms}. In \bibinfo{booktitle}{\emph{Proceedings of the 2017
  ACM Conference on Economics and Computation}} (Cambridge, Massachusetts, USA)
  \emph{(\bibinfo{series}{EC '17})}. \bibinfo{publisher}{Association for
  Computing Machinery}, \bibinfo{address}{New York, NY, USA},
  \bibinfo{pages}{589–590}.
\newblock
\showISBNx{9781450345279}


\bibitem[\protect\citeauthoryear{Cesa-Bianchi, Cesari, Colomboni, Fusco, and
  Leonardi}{Cesa-Bianchi et~al\mbox{.}}{2023a}]%
        {cesa2023bilateral}
\bibfield{author}{\bibinfo{person}{Nicol{\`o} Cesa-Bianchi},
  \bibinfo{person}{Tommaso Cesari}, \bibinfo{person}{Roberto Colomboni},
  \bibinfo{person}{Federico Fusco}, {and} \bibinfo{person}{Stefano Leonardi}.}
  \bibinfo{year}{2023}\natexlab{a}.
\newblock \showarticletitle{Bilateral trade: A regret minimization
  perspective}.
\newblock \bibinfo{journal}{\emph{Mathematics of Operations Research}}
  \bibinfo{volume}{0}, \bibinfo{number}{0} (\bibinfo{year}{2023}),
  \bibinfo{pages}{null}.
\newblock


\bibitem[\protect\citeauthoryear{Cesa-Bianchi, Cesari, Colomboni, Fusco, and
  Leonardi}{Cesa-Bianchi et~al\mbox{.}}{2021}]%
        {cesa2021regret}
\bibfield{author}{\bibinfo{person}{Nicol{\`o} Cesa-Bianchi},
  \bibinfo{person}{Tommaso~R Cesari}, \bibinfo{person}{Roberto Colomboni},
  \bibinfo{person}{Federico Fusco}, {and} \bibinfo{person}{Stefano Leonardi}.}
  \bibinfo{year}{2021}\natexlab{}.
\newblock \showarticletitle{A regret analysis of bilateral trade}. In
  \bibinfo{booktitle}{\emph{Proceedings of the 22nd ACM Conference on Economics
  and Computation}}. \bibinfo{publisher}{Association for Computing Machinery},
  \bibinfo{address}{USA}, \bibinfo{pages}{289--309}.
\newblock


\bibitem[\protect\citeauthoryear{Cesa-Bianchi, Cesari, Colomboni, Fusco, and
  Leonardi}{Cesa-Bianchi et~al\mbox{.}}{2023b}]%
        {cesa2023repeated}
\bibfield{author}{\bibinfo{person}{Nicol{\`o} Cesa-Bianchi},
  \bibinfo{person}{Tommaso~R Cesari}, \bibinfo{person}{Roberto Colomboni},
  \bibinfo{person}{Federico Fusco}, {and} \bibinfo{person}{Stefano Leonardi}.}
  \bibinfo{year}{2023}\natexlab{b}.
\newblock \showarticletitle{Repeated Bilateral Trade Against a Smoothed
  Adversary}. In \bibinfo{booktitle}{\emph{The Thirty Sixth Annual Conference
  on Learning Theory}}. PMLR, \bibinfo{publisher}{PMLR},
  \bibinfo{address}{USA}, \bibinfo{pages}{1095--1130}.
\newblock


\bibitem[\protect\citeauthoryear{Cesa-Bianchi and Lugosi}{Cesa-Bianchi and
  Lugosi}{2006}]%
        {cesa2006prediction}
\bibfield{author}{\bibinfo{person}{Nicolo Cesa-Bianchi} {and}
  \bibinfo{person}{G{\'a}bor Lugosi}.} \bibinfo{year}{2006}\natexlab{}.
\newblock \bibinfo{booktitle}{\emph{Prediction, learning, and games}}.
\newblock \bibinfo{publisher}{Cambridge University Press},
  \bibinfo{address}{UK}.
\newblock


\bibitem[\protect\citeauthoryear{Cesari and Colomboni}{Cesari and
  Colomboni}{2021}]%
        {cesari2021nearest}
\bibfield{author}{\bibinfo{person}{Tommaso~R Cesari} {and}
  \bibinfo{person}{Roberto Colomboni}.} \bibinfo{year}{2021}\natexlab{}.
\newblock \showarticletitle{A nearest neighbor characterization of {L}ebesgue
  points in metric measure spaces}.
\newblock \bibinfo{journal}{\emph{Mathematical Statistics and Learning}}
  \bibinfo{volume}{3}, \bibinfo{number}{1} (\bibinfo{year}{2021}),
  \bibinfo{pages}{71--112}.
\newblock


\bibitem[\protect\citeauthoryear{Colini{-}Baldeschi, de~Keijzer, Leonardi, and
  Turchetta}{Colini{-}Baldeschi et~al\mbox{.}}{2016}]%
        {Colini-Baldeschi16}
\bibfield{author}{\bibinfo{person}{Riccardo Colini{-}Baldeschi},
  \bibinfo{person}{Bart de Keijzer}, \bibinfo{person}{Stefano Leonardi}, {and}
  \bibinfo{person}{Stefano Turchetta}.} \bibinfo{year}{2016}\natexlab{}.
\newblock \showarticletitle{Approximately Efficient Double Auctions with Strong
  Budget Balance}. In \bibinfo{booktitle}{\emph{{ACM-SIAM} Symposium on
  Discrete Algorithms, {SODA}'16}}. \bibinfo{publisher}{{SIAM}},
  \bibinfo{address}{USA}, \bibinfo{pages}{1424--1443}.
\newblock


\bibitem[\protect\citeauthoryear{Colini{-}Baldeschi, Goldberg, de~Keijzer,
  Leonardi, and Turchetta}{Colini{-}Baldeschi et~al\mbox{.}}{2017}]%
        {Colini-Baldeschi17}
\bibfield{author}{\bibinfo{person}{Riccardo Colini{-}Baldeschi},
  \bibinfo{person}{Paul~W. Goldberg}, \bibinfo{person}{Bart de Keijzer},
  \bibinfo{person}{Stefano Leonardi}, {and} \bibinfo{person}{Stefano
  Turchetta}.} \bibinfo{year}{2017}\natexlab{}.
\newblock \showarticletitle{Fixed Price Approximability of the Optimal Gain
  from Trade}. In \bibinfo{booktitle}{\emph{Web and Internet Economics,
  {WINE}'17}} \emph{(\bibinfo{series}{Lecture Notes in Computer Science},
  Vol.~\bibinfo{volume}{10660})}. \bibinfo{publisher}{Springer},
  \bibinfo{address}{Germany}, \bibinfo{pages}{146--160}.
\newblock


\bibitem[\protect\citeauthoryear{Colini-Baldeschi, Goldberg, Keijzer, Leonardi,
  Roughgarden, and Turchetta}{Colini-Baldeschi et~al\mbox{.}}{2020}]%
        {colini2020approximately}
\bibfield{author}{\bibinfo{person}{Riccardo Colini-Baldeschi},
  \bibinfo{person}{Paul~W Goldberg}, \bibinfo{person}{Bart~de Keijzer},
  \bibinfo{person}{Stefano Leonardi}, \bibinfo{person}{Tim Roughgarden}, {and}
  \bibinfo{person}{Stefano Turchetta}.} \bibinfo{year}{2020}\natexlab{}.
\newblock \showarticletitle{Approximately efficient two-sided combinatorial
  auctions}.
\newblock \bibinfo{journal}{\emph{ACM Transactions on Economics and Computation
  (TEAC)}} \bibinfo{volume}{8}, \bibinfo{number}{1} (\bibinfo{year}{2020}),
  \bibinfo{pages}{1--29}.
\newblock


\bibitem[\protect\citeauthoryear{Deng, Mao, Sivan, and Wang}{Deng
  et~al\mbox{.}}{2022}]%
        {DengMSW21}
\bibfield{author}{\bibinfo{person}{Yuan Deng}, \bibinfo{person}{Jieming Mao},
  \bibinfo{person}{Balasubramanian Sivan}, {and} \bibinfo{person}{Kangning
  Wang}.} \bibinfo{year}{2022}\natexlab{}.
\newblock \showarticletitle{Approximately efficient bilateral trade}. In
  \bibinfo{booktitle}{\emph{{STOC}}}. \bibinfo{publisher}{{ACM}},
  \bibinfo{address}{Italy}, \bibinfo{pages}{718--721}.
\newblock


\bibitem[\protect\citeauthoryear{D\"{u}tting, Fusco, Lazos, Leonardi, and
  Reiffenh\"{a}user}{D\"{u}tting et~al\mbox{.}}{2021}]%
        {dutting2021efficient}
\bibfield{author}{\bibinfo{person}{Paul D\"{u}tting}, \bibinfo{person}{Federico
  Fusco}, \bibinfo{person}{Philip Lazos}, \bibinfo{person}{Stefano Leonardi},
  {and} \bibinfo{person}{Rebecca Reiffenh\"{a}user}.}
  \bibinfo{year}{2021}\natexlab{}.
\newblock \showarticletitle{Efficient Two-Sided Markets with Limited
  Information}. In \bibinfo{booktitle}{\emph{Proceedings of the 53rd Annual ACM
  SIGACT Symposium on Theory of Computing}} (Virtual, Italy)
  \emph{(\bibinfo{series}{STOC 2021})}. \bibinfo{publisher}{Association for
  Computing Machinery}, \bibinfo{address}{New York, NY, USA},
  \bibinfo{pages}{1452–1465}.
\newblock
\showISBNx{9781450380539}


\bibitem[\protect\citeauthoryear{for International~Settlements}{for
  International~Settlements}{2022}]%
        {bis2023}
\bibfield{author}{\bibinfo{person}{Bank for International~Settlements}.}
  \bibinfo{year}{2022}\natexlab{}.
\newblock \bibinfo{title}{OTC derivatives statistics at end-June 2022}.
\newblock
\newblock
\urldef\tempurl%
\url{https://www.bis.org/publ/otc_hy2211.pdf}
\showURL{%
\tempurl}


\bibitem[\protect\citeauthoryear{Kang, Pernice, and Vondr{\'a}k}{Kang
  et~al\mbox{.}}{2022}]%
        {kang2022fixed}
\bibfield{author}{\bibinfo{person}{Zi~Yang Kang}, \bibinfo{person}{Francisco
  Pernice}, {and} \bibinfo{person}{Jan Vondr{\'a}k}.}
  \bibinfo{year}{2022}\natexlab{}.
\newblock \showarticletitle{Fixed-price approximations in bilateral trade}. In
  \bibinfo{booktitle}{\emph{Proceedings of the 2022 Annual ACM-SIAM Symposium
  on Discrete Algorithms (SODA)}}. SIAM, \bibinfo{publisher}{Society for
  Industrial and Applied Mathematics}, \bibinfo{address}{Alexandria, VA, USA},
  \bibinfo{pages}{2964--2985}.
\newblock


\bibitem[\protect\citeauthoryear{Lucas~Jr}{Lucas~Jr}{1989}]%
        {lucas1989effects}
\bibfield{author}{\bibinfo{person}{Robert~E Lucas~Jr}.}
  \bibinfo{year}{1989}\natexlab{}.
\newblock \bibinfo{title}{The effects of monetary shocks when prices are set in
  advance}.
\newblock
\newblock


\bibitem[\protect\citeauthoryear{Myerson and Satterthwaite}{Myerson and
  Satterthwaite}{1983}]%
        {myerson1983efficient}
\bibfield{author}{\bibinfo{person}{Roger~B Myerson} {and}
  \bibinfo{person}{Mark~A Satterthwaite}.} \bibinfo{year}{1983}\natexlab{}.
\newblock \showarticletitle{Efficient mechanisms for bilateral trading}.
\newblock \bibinfo{journal}{\emph{Journal of economic theory}}
  \bibinfo{volume}{29}, \bibinfo{number}{2} (\bibinfo{year}{1983}),
  \bibinfo{pages}{265--281}.
\newblock


\bibitem[\protect\citeauthoryear{Sherstyuk, Phankitnirundorn, and
  Roberts}{Sherstyuk et~al\mbox{.}}{2020}]%
        {sherstyuk2020randomized}
\bibfield{author}{\bibinfo{person}{Katerina Sherstyuk}, \bibinfo{person}{Krit
  Phankitnirundorn}, {and} \bibinfo{person}{Michael~J Roberts}.}
  \bibinfo{year}{2020}\natexlab{}.
\newblock \showarticletitle{Randomized double auctions: gains from trade,
  trader roles, and price discovery}.
\newblock \bibinfo{journal}{\emph{Experimental Economics}}
  \bibinfo{volume}{24}, \bibinfo{number}{4} (\bibinfo{year}{2020}),
  \bibinfo{pages}{1--40}.
\newblock


\bibitem[\protect\citeauthoryear{Weill}{Weill}{2020}]%
        {weill2020search}
\bibfield{author}{\bibinfo{person}{Pierre-Olivier Weill}.}
  \bibinfo{year}{2020}\natexlab{}.
\newblock \showarticletitle{The search theory of over-the-counter markets}.
\newblock \bibinfo{journal}{\emph{Annual Review of Economics}}
  \bibinfo{volume}{12} (\bibinfo{year}{2020}), \bibinfo{pages}{747--773}.
\newblock


\bibitem[\protect\citeauthoryear{Williams}{Williams}{1991}]%
        {williams1991probability}
\bibfield{author}{\bibinfo{person}{David Williams}.}
  \bibinfo{year}{1991}\natexlab{}.
\newblock \bibinfo{booktitle}{\emph{Probability with martingales}}.
\newblock \bibinfo{publisher}{Cambridge university press},
  \bibinfo{address}{UK}.
\newblock


\end{thebibliography}


\cleardoublepage
\appendix

\newpage

\onecolumn

\section{Proof of the Representation Lemma}

We will prove now that, if \(M_{1},M_{2}\) are two independent random
variables supported in \(\left\lbrack 0,1 \right\rbrack\) that share the
same distribution \(\mu\) and (hence have) common expectation
\(\bar{\mu}\), then, for each \(t\in\mathbb{ N}\) and
each \(p \in \left\lbrack 0,1 \right\rbrack\) it holds that:

\[\mathbb{E}\left\lbrack \left( M_{1} \vee M_{2} - M_{1} \land M_{2} \right)\mathbb{I}\left\{ M_{1} \land M_{2} \leq p \leq M_{1} \vee M_{2} \right\} \right\rbrack = \tilde{\rho}\left( \mu \right)\left( p \right) + \mu\left\{ p \right\}\left( \int_{0}^{p}{\mu\left\lbrack 0,\lambda \right\rbrack \,\mathrm{d}\lambda} + \int_{p}^{1}{\mu\left\lbrack \lambda,1 \right\rbrack \,\mathrm{d}\lambda} \right)\]
where
\[\tilde{\rho}\left( \mu \right)\left( p \right) = \int_{0}^{p}{\left( \mu\left\lbrack 0,\lambda \right\rbrack + \mu\left\lbrack 0,\lambda \right) \right)\,\mathrm{d}\lambda} + \left( \mu\left\lbrack 0,p \right\rbrack + \mu\left\lbrack 0,p \right) \right)\left( \bar{\mu} - p \right).\]
For notational convenience, let \(M\) be another random variable with
distribution \(\mu\).

In what follows, we will use the following observation. 
For any
\(0 \leq a < b \leq 1\) we have
\begin{itemize}
\item
  \(\int_{a}^{b}{\mathbb{P}\left\lbrack M \geq \lambda \right\rbrack \,\mathrm{d}\lambda} = \int_{a}^{b}{\mathbb{P}\left\lbrack M > \lambda \right\rbrack \,\mathrm{d}\lambda},\)
\item
  \(\int_{a}^{b}{\mathbb{P}\left\lbrack M \leq \lambda \right\rbrack \,\mathrm{d}\lambda} = \int_{a}^{b}{\mathbb{P}\left\lbrack M < \lambda \right\rbrack \,\mathrm{d}\lambda}.\)
\end{itemize}
This is due to the fact that the two functions
\(\lambda \longmapsto \mathbb{P}\left\lbrack M \geq \lambda \right\rbrack\)
and
\(\lambda \longmapsto \mathbb{P}\left\lbrack M > \lambda \right\rbrack\)
are different only in a set that is at most countable. Hence, the set
where they differ have measure zero and the first two integral
coincides. The same reasoning applies to the second two integrals.

Now, notice that, for all $p\in[0,1]$
\begin{align*}
&    
\left( M_{1} \vee M_{2} - M_{1} \land M_{2} \right)\mathbb{I}\left\{ M_{1} \land M_{2} \leq p \leq M_{1} \vee M_{2} \right\} = \int_{0}^{1}{\mathbb{I}\left\{ M_{1} \land M_{2} \leq \lambda \leq M_{1} \vee M_{2} \right\} \,\mathrm{d}\lambda} \cdot \mathbb{I}\left\{ M_{1} \land M_{2} \leq p \leq M_{1} \vee M_{2} \right\}
\\
& 
\qquad
= \int_{0}^{p}{\mathbb{I}\left\{ M_{1} \land M_{2} \leq \lambda \cap p \leq M_{1} \vee M_{2} \right\} \,\mathrm{d}\lambda} + \int_{p}^{1}{\mathbb{I}\left\{ M_{1} \land M_{2} \leq p \cap \lambda \leq M_{1} \vee M_{2} \right\} \,\mathrm{d}\lambda}
\\
&
\qquad
= \int_{0}^{p}{\left( \mathbb{I}\left\{ M_{1} \leq \lambda \cap p \leq M_{2} \right\} + \mathbb{I}\left\{ M_{2} \leq \lambda \cap p \leq M_{1} \right\} \right)\,\mathrm{d}\lambda} + \int_{p}^{1}{\left( \mathbb{I}\left\{ M_{1} \leq p \cap \lambda \leq M_{2} \right\} + 
 \mathbb{I}\left\{ M_{2} \leq p \cap \lambda \leq M_{1} \right\} \right)\,\mathrm{d}\lambda}.
\end{align*}

It follows that, for each $p\in[0,1]$,
{\allowdisplaybreaks
\begin{align*}
&
\mathbb{E}\left\lbrack \left( M_{1} \vee M_{2} - M_{1} \land M_{2} \right)\mathbb{I}\left\{ M_{1} \land M_{2} \leq p \leq M_{1} \vee M_{2} \right\} \right\rbrack 
\\
&
\qquad
= \mathbb{E}\left\lbrack \int_{0}^{p}{\left( \mathbb{I}\left\{ M_{1} \leq \lambda \cap p \leq M_{2} \right\} + \mathbb{I}\left\{ M_{2} \leq \lambda \cap p \leq M_{1} \right\} \right)\,\mathrm{d}\lambda} + \int_{p}^{1}{\left( \mathbb{I}\left\{ M_{1} \leq p \cap \lambda \leq M_{2} \right\}+\mathbb{ I}\left\{ M_{2} \leq p \cap \lambda \leq M_{1} \right\} \right)\,\mathrm{d}\lambda} \right\rbrack
\\
&
\qquad
= \int_{0}^{p}{\left( \mathbb{P}\left\lbrack M_{1} \leq \lambda \cap p \leq M_{2} \right\rbrack + \mathbb{P}\left\lbrack M_{2} \leq \lambda \cap p \leq M_{1} \right\rbrack \right)\,\mathrm{d}\lambda} + \int_{p}^{1}{\left( \mathbb{P}\left\lbrack M_{1} \leq p \cap \lambda \leq M_{2} \right\rbrack+\mathbb{ P}\left\lbrack M_{2} \leq p \cap \lambda \leq M_{1} \right\rbrack \right)\,\mathrm{d}\lambda}
\\
&
\qquad
= \int_{0}^{p}{\left( \mathbb{P}\left\lbrack M_{1} \leq \lambda \right\rbrack\mathbb{P}\left\lbrack p \leq M_{2} \right\rbrack+\mathbb{ P}\left\lbrack M_{2} \leq \lambda \right\rbrack\mathbb{P}\left\lbrack p \leq M_{1} \right\rbrack \right)\,\mathrm{d}\lambda} + \int_{p}^{1}{\left( \mathbb{P}\left\lbrack M_{1} \leq p \right\rbrack\mathbb{P}\left\lbrack \lambda \leq M_{2} \right\rbrack+\mathbb{ P}\left\lbrack M_{2} \leq p \right\rbrack\mathbb{P}\left\lbrack \lambda \leq M_{1} \right\rbrack \right)\,\mathrm{d}\lambda}
\\
&
\qquad
= 2\int_{0}^{p}{\mathbb{P}\left\lbrack M \leq \lambda \right\rbrack\mathbb{P}\left\lbrack p \leq M \right\rbrack \,\mathrm{d}\lambda} + 2\int_{p}^{1}{\mathbb{P}\left\lbrack M \leq p \right\rbrack\mathbb{P}\left\lbrack \lambda \leq M \right\rbrack \,\mathrm{d}\lambda}
\\
&
\qquad
= 2\mathbb{P}\left\lbrack M \leq p \right\rbrack\int_{p}^{1}{\mathbb{P}\left\lbrack \lambda \leq M \right\rbrack \,\mathrm{d}\lambda} + 2\mathbb{P}\left\lbrack M \geq p \right\rbrack\int_{0}^{p}{\mathbb{P}\left\lbrack M \leq \lambda \right\rbrack \,\mathrm{d}\lambda}
\\
&
\qquad
= \left( \mathbb{P}\left\lbrack M \leq p \right\rbrack+\mathbb{ P}\left\lbrack M < p \right\rbrack \right)\int_{p}^{1}{\mathbb{P}\left\lbrack M \geq \lambda \right\rbrack \,\mathrm{d}\lambda} + \left( \mathbb{P}\left\lbrack M \geq p \right\rbrack+\mathbb{ P}\left\lbrack M > p \right\rbrack \right)\int_{0}^{p}{\mathbb{P}\left\lbrack M \leq \lambda \right\rbrack \,\mathrm{d}\lambda} 
\\
&
\qquad\qquad
+ \left( \mathbb{P}\left\lbrack M \leq p \right\rbrack-\mathbb{ P}\left\lbrack M < p \right\rbrack \right)\int_{p}^{1}{\mathbb{P}\left\lbrack M \geq \lambda \right\rbrack \,\mathrm{d}\lambda} + \left( \mathbb{P}\left\lbrack M \geq p \right\rbrack-\mathbb{ P}\left\lbrack M > p \right\rbrack \right)\int_{0}^{p}{\mathbb{P}\left\lbrack M \leq \lambda \right\rbrack \,\mathrm{d}\lambda}
\\
&
\qquad
= \left( \mathbb{P}\left\lbrack M \leq p \right\rbrack+\mathbb{ P}\left\lbrack M < p \right\rbrack \right)\mathbb{E}\left\lbrack M \right\rbrack - \left( \mathbb{P}\left\lbrack M \leq p \right\rbrack+\mathbb{ P}\left\lbrack M < p \right\rbrack \right)\int_{0}^{p}{\mathbb{P}\left\lbrack M \geq \lambda \right\rbrack \,\mathrm{d}\lambda}
\\
&
\qquad\qquad
+ \left( 1 - \mathbb{P}\left\lbrack M < p \right\rbrack + 1 - \mathbb{P}\left\lbrack M \leq p \right\rbrack \right)\int_{0}^{p}{\left( 1 - \mathbb{P}\left\lbrack M > \lambda \right\rbrack \right)\,\mathrm{d}\lambda} + \left( \mathbb{P}\left\lbrack M \leq p \right\rbrack-\mathbb{ P}\left\lbrack M < p \right\rbrack \right)\int_{p}^{1}{\mathbb{P}\left\lbrack M \geq \lambda \right\rbrack \,\mathrm{d}\lambda}
\\
&
\qquad\qquad
+ \left( \mathbb{P}\left\lbrack M \geq p \right\rbrack-\mathbb{ P}\left\lbrack M > p \right\rbrack \right)\int_{0}^{p}{\mathbb{P}\left\lbrack M \leq \lambda \right\rbrack \,\mathrm{d}\lambda}
\\
&
\qquad
= \left( \mathbb{P}\left\lbrack M \leq p \right\rbrack+\mathbb{ P}\left\lbrack M < p \right\rbrack \right)\left( \mathbb{E}\left\lbrack M \right\rbrack - p \right) - \left( \mathbb{P}\left\lbrack M \leq p \right\rbrack+\mathbb{ P}\left\lbrack M < p \right\rbrack \right)\int_{0}^{p}{\mathbb{P}\left\lbrack M \geq \lambda \right\rbrack \,\mathrm{d}\lambda} 
\\
&
\qquad\qquad
+ 2p - \left( 1 - \mathbb{P}\left\lbrack M < p \right\rbrack + 1 - \mathbb{P}\left\lbrack M \leq p \right\rbrack \right)\int_{0}^{p}{\mathbb{P}\left\lbrack M > \lambda \right\rbrack \,\mathrm{d}\lambda} + \left( \mathbb{P}\left\lbrack M \leq p \right\rbrack-\mathbb{ P}\left\lbrack M < p \right\rbrack \right)\int_{p}^{1}{\mathbb{P}\left\lbrack M \geq \lambda \right\rbrack \,\mathrm{d}\lambda} 
\\
&
\qquad\qquad
+ \left( \mathbb{P}\left\lbrack M \geq p \right\rbrack-\mathbb{ P}\left\lbrack M > p \right\rbrack \right)\int_{0}^{p}{\mathbb{P}\left\lbrack M \leq \lambda \right\rbrack \,\mathrm{d}\lambda}
\\
&
\qquad
= \left( \mathbb{P}\left\lbrack M \leq p \right\rbrack+\mathbb{ P}\left\lbrack M < p \right\rbrack \right)\left( \mathbb{E}\left\lbrack M \right\rbrack - p \right) - \left( \mathbb{P}\left\lbrack M \leq p \right\rbrack+\mathbb{ P}\left\lbrack M < p \right\rbrack \right)\int_{0}^{p}{\mathbb{P}\left\lbrack M \geq \lambda \right\rbrack \,\mathrm{d}\lambda} 
\\
&
\qquad\qquad
+ 2p - \left( 1 - \mathbb{P}\left\lbrack M < p \right\rbrack + 1 - \mathbb{P}\left\lbrack M \leq p \right\rbrack \right)\int_{0}^{p}{\left( 1 - \mathbb{P}\left\lbrack M \leq \lambda \right\rbrack \right)\,\mathrm{d}\lambda} 
\\
&
\qquad\qquad
+ \left( \mathbb{P}\left\lbrack M \leq p \right\rbrack-\mathbb{ P}\left\lbrack M < p \right\rbrack \right)\int_{p}^{1}{\mathbb{P}\left\lbrack M \geq \lambda \right\rbrack \,\mathrm{d}\lambda} + \left( \mathbb{P}\left\lbrack M \geq p \right\rbrack-\mathbb{ P}\left\lbrack M > p \right\rbrack \right)\int_{0}^{p}{\mathbb{P}\left\lbrack M \leq \lambda \right\rbrack \,\mathrm{d}\lambda}
\\
&
\qquad
= 2\int_{0}^{p}{\mathbb{P}\left\lbrack M \leq \lambda \right\rbrack \,\mathrm{d}\lambda} + \left( \mathbb{P}\left\lbrack M \leq p \right\rbrack+\mathbb{ P}\left\lbrack M < p \right\rbrack \right)\left( \mathbb{E}\left\lbrack M \right\rbrack - p \right) - \left( \mathbb{P}\left\lbrack M < p \right\rbrack+\mathbb{ P}\left\lbrack M \leq p \right\rbrack \right)\int_{0}^{p}{\mathbb{P}\left\lbrack M \leq \lambda \right\rbrack \,\mathrm{d}\lambda}
\\
&
\qquad\qquad
- \left( \mathbb{P}\left\lbrack M \leq p \right\rbrack+\mathbb{ P}\left\lbrack M < p \right\rbrack \right)\int_{0}^{p}{\mathbb{P}\left\lbrack M \geq \lambda \right\rbrack \,\mathrm{d}\lambda} + \left( \mathbb{P}\left\lbrack M < p \right\rbrack+\mathbb{ P}\left\lbrack M \leq p \right\rbrack \right)p 
\\
&
\qquad\qquad
+ \left( \mathbb{P}\left\lbrack M \leq p \right\rbrack-\mathbb{ P}\left\lbrack M < p \right\rbrack \right)\int_{p}^{1}{\mathbb{P}\left\lbrack M \geq \lambda \right\rbrack \,\mathrm{d}\lambda} + \left( \mathbb{P}\left\lbrack M \geq p \right\rbrack-\mathbb{ P}\left\lbrack M > p \right\rbrack \right)\int_{0}^{p}{\mathbb{P}\left\lbrack M \leq \lambda \right\rbrack \,\mathrm{d}\lambda}
\\
&
\qquad
= \int_{0}^{p}{\left( \mu\left\lbrack 0,\lambda \right\rbrack + \mu\left\lbrack 0,\lambda \right) \right)\,\mathrm{d}\lambda} + \left( \mu\left\lbrack 0,p \right\rbrack + \mu\left\lbrack 0,p \right) \right)\left( \bar{\mu} - p \right) 
- \left( \mathbb{P}\left\lbrack M < p \right\rbrack+\mathbb{ P}\left\lbrack M \leq p \right\rbrack \right)\int_{0}^{p}{\mathbb{P}\left\lbrack M \leq \lambda \right\rbrack \,\mathrm{d}\lambda}
\\
&
\qquad\qquad
- \left( \mathbb{P}\left\lbrack M \leq p \right\rbrack+\mathbb{ P}\left\lbrack M < p \right\rbrack \right)\int_{0}^{p}{\mathbb{P}\left\lbrack M \geq \lambda \right\rbrack \,\mathrm{d}\lambda} + \left( \mathbb{P}\left\lbrack M < p \right\rbrack+\mathbb{ P}\left\lbrack M \leq p \right\rbrack \right)p 
\\
&
\qquad\qquad
+ \left( \mathbb{P}\left\lbrack M \leq p \right\rbrack-\mathbb{ P}\left\lbrack M < p \right\rbrack \right)\int_{p}^{1}{\mathbb{P}\left\lbrack M \geq \lambda \right\rbrack \,\mathrm{d}\lambda} + \left( \mathbb{P}\left\lbrack M \geq p \right\rbrack-\mathbb{ P}\left\lbrack M > p \right\rbrack \right)\int_{0}^{p}{\mathbb{P}\left\lbrack M \leq \lambda \right\rbrack \,\mathrm{d}\lambda}
\\
&
\qquad
\eqqcolon \tilde{\rho}\left( \mu \right)\left( p \right) + \left( \mathrm I \right).
\end{align*}
}%
It is left to prove that
\(\left( \mathrm I \right) = \ \mu\left\{ p \right\}\left( \int_{0}^{p}{\mu\left\lbrack 0,\lambda \right\rbrack \,\mathrm{d}\lambda} + \int_{p}^{1}{\mu\left\lbrack \lambda,1 \right\rbrack \,\mathrm{d}\lambda} \right)\).
In fact
{\allowdisplaybreaks
\begin{align*}
\left( \mathrm I \right)
& = - \left( \mathbb{P}\left\lbrack M < p \right\rbrack+\mathbb{ P}\left\lbrack M \leq p \right\rbrack \right)\int_{0}^{p}{\mathbb{P}\left\lbrack M \leq \lambda \right\rbrack \,\mathrm{d}\lambda} - \left( \mathbb{P}\left\lbrack M \leq p \right\rbrack+\mathbb{ P}\left\lbrack M < p \right\rbrack \right)\int_{0}^{p}{\mathbb{P}\left\lbrack M \geq \lambda \right\rbrack \,\mathrm{d}\lambda} + \left( \mathbb{P}\left\lbrack M < p \right\rbrack+\mathbb{ P}\left\lbrack M \leq p \right\rbrack \right)p 
\\
&
\qquad
+ \left( \mathbb{P}\left\lbrack M \leq p \right\rbrack-\mathbb{ P}\left\lbrack M < p \right\rbrack \right)\int_{p}^{1}{\mathbb{P}\left\lbrack M \geq \lambda \right\rbrack \,\mathrm{d}\lambda} + \left( \mathbb{P}\left\lbrack M \geq p \right\rbrack-\mathbb{ P}\left\lbrack M > p \right\rbrack \right)\int_{0}^{p}{\mathbb{P}\left\lbrack M \leq \lambda \right\rbrack \,\mathrm{d}\lambda}
\\
&
= - \left( \mathbb{P}\left\lbrack M < p \right\rbrack+\mathbb{ P}\left\lbrack M \leq p \right\rbrack \right)\int_{0}^{p}{\mathbb{P}\left\lbrack M \leq \lambda \right\rbrack \,\mathrm{d}\lambda} - \left( \mathbb{P}\left\lbrack M \leq p \right\rbrack+\mathbb{ P}\left\lbrack M < p \right\rbrack \right)\int_{0}^{p}{\left( 1 - \mathbb{P}\left\lbrack M \leq \lambda \right\rbrack \right)\,\mathrm{d}\lambda} 
\\
&
\qquad
+ \left( \mathbb{P}\left\lbrack M < p \right\rbrack+\mathbb{ P}\left\lbrack M \leq p \right\rbrack \right)p 
+ \left( \mathbb{P}\left\lbrack M \leq p \right\rbrack-\mathbb{ P}\left\lbrack M < p \right\rbrack \right)\int_{p}^{1}{\mathbb{P}\left\lbrack M \geq \lambda \right\rbrack \,\mathrm{d}\lambda} 
\\
&
\qquad
+ \left( \mathbb{P}\left\lbrack M \geq p \right\rbrack-\mathbb{ P}\left\lbrack M > p \right\rbrack \right)\int_{0}^{p}{\mathbb{P}\left\lbrack M \leq \lambda \right\rbrack \,\mathrm{d}\lambda}
\\
&
= \brb{ \left( \mathbb{P}\left\lbrack M \leq p \right\rbrack+\mathbb{ P}\left\lbrack M < p \right\rbrack \right) - \left( \mathbb{P}\left\lbrack M < p \right\rbrack+\mathbb{ P}\left\lbrack M \leq p \right\rbrack \right) } \int_{0}^{p}{\mathbb{P}\left\lbrack M \leq \lambda \right\rbrack \,\mathrm{d}\lambda} 
\\
&
\qquad
- \left( \mathbb{P}\left\lbrack M \leq p \right\rbrack+\mathbb{ P}\left\lbrack M < p \right\rbrack \right)p + \left( \mathbb{P}\left\lbrack M < p \right\rbrack+\mathbb{ P}\left\lbrack M \leq p \right\rbrack \right)p 
\\
&
\qquad
+ \left( \mathbb{P}\left\lbrack M \leq p \right\rbrack-\mathbb{ P}\left\lbrack M < p \right\rbrack \right)\int_{p}^{1}{\mathbb{P}\left\lbrack M \geq \lambda \right\rbrack \,\mathrm{d}\lambda} + \left( \mathbb{P}\left\lbrack M \geq p \right\rbrack-\mathbb{ P}\left\lbrack M > p \right\rbrack \right)\int_{0}^{p}{\mathbb{P}\left\lbrack M \leq \lambda \right\rbrack \,\mathrm{d}\lambda}
\\
&
= \left( \mathbb{P}\left\lbrack M \leq p \right\rbrack-\mathbb{ P}\left\lbrack M < p \right\rbrack \right)\int_{p}^{1}{\mathbb{P}\left\lbrack M \geq \lambda \right\rbrack \,\mathrm{d}\lambda} + \left( \mathbb{P}\left\lbrack M \geq p \right\rbrack-\mathbb{ P}\left\lbrack M > p \right\rbrack \right)\int_{0}^{p}{\mathbb{P}\left\lbrack M \leq \lambda \right\rbrack \,\mathrm{d}\lambda}
\\
&
= \mathbb{P}\left\lbrack M = p \right\rbrack\int_{p}^{1}{\mathbb{P}\left\lbrack M \geq \lambda \right\rbrack \,\mathrm{d}\lambda} + \mathbb{P}\left\lbrack M = p \right\rbrack\int_{0}^{p}{\mathbb{P}\left\lbrack M \leq \lambda \right\rbrack \,\mathrm{d}\lambda}
\\
&
= \mathbb{P}\left\lbrack M = p \right\rbrack\left( \int_{0}^{p}{\mathbb{P}\left\lbrack M \leq \lambda \right\rbrack \,\mathrm{d}\lambda} + \int_{p}^{1}{\mathbb{P}\left\lbrack M \geq \lambda \right\rbrack \,\mathrm{d}\lambda} \right)
\\
&
= \mu\left\{ p \right\}\left( \int_{0}^{p}{\mu\left\lbrack 0,\lambda \right\rbrack \,\mathrm{d}\lambda} + \int_{p}^{1}{\mu\left\lbrack \lambda,1 \right\rbrack \,\mathrm{d}\lambda} \right)
\end{align*}
}
which concludes the proof of the claim.

\section{Missing Details in the Proof of Theorem 3.2}
We will show now that, with the notation of the proof of the Theorem
3.2, for any \(M \geq 2\), if \(t \geq 580M^{4}\), it holds that

\[\mathbb{E}\left\lbrack \left( Z-\mathbb{ E}\left\lbrack Z\mid B_{Z,1},\ldots,B_{Z,2\left( t - 1 \right)} \right\rbrack \right)^{2} \right\rbrack \geq \frac{1}{147} \cdot \frac{1}{t - 1}.\]
For any \(t\in\mathbb{ N}\), we have
\begin{align*}
&
\mathbb{E}\left\lbrack \left( Z-\mathbb{ E}\left\lbrack Z\mid B_{Z,1},\ldots,B_{Z,t} \right\rbrack \right)^{2} \right\rbrack\geq\mathbb{ E}\left\lbrack \left( Z-\mathbb{ E}\left\lbrack Z\mid B_{Z,1},\ldots,B_{Z,t} \right\rbrack \right)^{2}\mathbb{I}\left\{ Z \in \left\lbrack \frac{1}{2} - \frac{\varepsilon_{M}}{9},\frac{1}{2} + \frac{\varepsilon_{M}}{9} \right\rbrack \right\} \right\rbrack
\\
&
\qquad
=\mathbb{ E}\left\lbrack \left( \underbrace{\lrb{Z-\sum_{k=1}^t B_{Z,k} }}_{a} + \underbrace{ \lrb{ \sum_{k=1}^t B_{Z,k} - \E[Z\mid B_{Z,1}, \dots, B_{Z,t}] } }_{b} \right)^{2}\mathbb{I}\left\{ Z \in \left\lbrack \frac{1}{2} - \frac{\varepsilon_{M}}{9},\frac{1}{2} + \frac{\varepsilon_{M}}{9} \right\rbrack \right\} \right\rbrack
\\
&
\qquad
\geq\mathbb{ E}\left\lbrack \left( Z - \frac{1}{t}\sum_{k = 1}^{t}B_{Z,k} \right)^{2}\mathbb{I}\left\{ Z \in \left\lbrack \frac{1}{2} - \frac{\varepsilon_{M}}{9},\frac{1}{2} + \frac{\varepsilon_{M}}{9} \right\rbrack \right\} \right\rbrack 
\\
&
\qquad\qquad
- 2\mathbb{E}\left\lbrack \vphantom{\frac{\sum}{\displaystyle\sum}}\left| Z - \frac{1}{t}\sum_{k = 1}^{t}B_{Z,k} \right|\left| \frac{1}{t}\sum_{k = 1}^{t}B_{Z,k}-\mathbb{ E}\left\lbrack Z\mid B_{Z,1},\ldots,B_{Z,t} \right\rbrack \right|\mathbb{I}\left\{ Z \in \left\lbrack \frac{1}{2} - \frac{\varepsilon_{M}}{9},\frac{1}{2} + \frac{\varepsilon_{M}}{9} \right\rbrack \right\} \right\rbrack
\eqqcolon
\left( \mathrm I \right) - 2 \cdot \left( \mathrm{II} \right),
\end{align*}
where the last inequality follows from
\(\left( a + b \right)^{2} \geq a^{2} - 2\left| ab\right|\).
Now, if \(W\) is a uniform random variable on
\(\left\lbrack \frac{1}{2} - \frac{\varepsilon_{M}}{9},\frac{1}{2} + \frac{\varepsilon_{M}}{9} \right\rbrack\)
independent of
\(\left( B_{q,t} \right)_{q \in \left\lbrack 0,1 \right\rbrack,t\in\mathbb{ N}}\),
we have that
\begin{align*}
&
\left( \mathrm I \right)=\mathbb{ E}\left\lbrack \left( Z - \frac{1}{t}\sum_{k = 1}^{t}B_{Z,k} \right)^{2}\mid Z \in \left\lbrack \frac{1}{2} - \frac{\varepsilon_{M}}{9},\frac{1}{2} + \frac{\varepsilon_{M}}{9} \right\rbrack \right\rbrack\mathbb{P}\left\lbrack Z \in \left\lbrack \frac{1}{2} - \frac{\varepsilon_{M}}{9},\frac{1}{2} + \frac{\varepsilon_{M}}{9} \right\rbrack \right\rbrack 
\\
&
\qquad
= \frac{1}{9}\mathbb{E}\left\lbrack \left( Z - \frac{1}{t}\sum_{k = 1}^{t}B_{Z,k} \right)^{2}\mid Z \in \left\lbrack \frac{1}{2} - \frac{\varepsilon_{M}}{9},\frac{1}{2} + \frac{\varepsilon_{M}}{9} \right\rbrack \right\rbrack 
= \frac{1}{9}\mathbb{E}\left\lbrack \left( W - \frac{1}{t}\sum_{k = 1}^{t}B_{W,k} \right)^{2} \right\rbrack \eqqcolon \left( \star \right).
\end{align*}
It follows that
\begin{align*}
\left( \star \right) 
&= \frac{1}{9}\int_{\frac{1}{2} - \frac{\varepsilon_{M}}{9}}^{\frac{1}{2} + \frac{\varepsilon_{M}}{9}}{\mathbb{E}\left\lbrack \left( w - \frac{1}{t}\sum_{k = 1}^{t}B_{w,k} \right)^{2} \right\rbrack \,\mathrm d\mathbb{P}_{W}\left( w \right)} = \frac{1}{9}\int_{\frac{1}{2} - \frac{\varepsilon_{M}}{9}}^{\frac{1}{2} + \frac{\varepsilon_{M}}{9}}{\operatorname{Var}\left\lbrack \frac{1}{t}\sum_{k = 1}^{t}B_{w,k} \right\rbrack \, \mathrm d\mathbb{P}_{W}\left( w \right)} 
\\
&
= \frac{1}{9}\int_{\frac{1}{2} - \frac{\varepsilon_{M}}{9}}^{\frac{1}{2} + \frac{\varepsilon_{M}}{9}}{\frac{w\left( 1 - w \right)}{t}\, \mathrm d\mathbb{P}_{W}\left( w \right)} \leq \frac{1}{9}\frac{3}{7}\frac{4}{7}\frac{1}{t} = \frac{4}{147} \cdot \frac{1}{t}.
\end{align*}
About the term \(\left( \text{II} \right)\), we have
\begin{align*}
&
\left( \mathrm{II} \right)\mathbb{\leq P}\left\lbrack Z \in \left\lbrack \frac{1}{2} - \frac{\varepsilon_{M}}{9},\frac{1}{2} + \frac{\varepsilon_{M}}{9} \right\rbrack \cap \frac{1}{t}\sum_{k = 1}^{t}B_{Z,k} \notin \left\lbrack \frac{1}{2} - \frac{\varepsilon_{M}}{6},\frac{1}{2} + \frac{\varepsilon_{M}}{6} \right\rbrack \right\rbrack
\\
&
\qquad
+\mathbb{ E}\left\lbrack \vphantom{\frac{\sum}{\displaystyle\sum}} \left| Z - \frac{1}{t}\sum_{k = 1}^{t}B_{Z,k} \right|\left| \frac{1}{t}\sum_{k = 1}^{t}B_{Z,k}-\mathbb{ E}\left\lbrack Z\mid B_{Z,1},\ldots,B_{Z,t} \right\rbrack \right| \cdot \right.
\\
&
\qquad\qquad
\cdot
\left.
\vphantom{\frac{\sum}{\displaystyle\sum}}
\mathbb{I}\left\{ Z \in \left\lbrack \frac{1}{2} - \frac{\varepsilon_{M}}{9},\frac{1}{2} + \frac{\varepsilon_{M}}{9} \right\rbrack \right\}\mathbb{I}\left\{ \frac{1}{t}\sum_{k = 1}^{t}B_{Z,k} \in \left\lbrack \frac{1}{2} - \frac{\varepsilon_{M}}{6},\frac{1}{2} + \frac{\varepsilon_{M}}{6} \right\rbrack\  \right\} \right\rbrack
\eqqcolon
\left( \text{III} \right) + \left( \text{IV} \right)
\end{align*}
About the term \(\left( \text{III} \right)\), we have
{
\allowdisplaybreaks
\begin{align*}
\left( \text{III} \right) 
& = \int_{\frac{1}{2} - \frac{\varepsilon_{M}}{9}}^{\frac{1}{2} + \frac{\varepsilon_{M}}{9}}{\mathbb{P}\left\lbrack \frac{1}{t}\sum_{k = 1}^{t}B_{z,k} \notin \left\lbrack \frac{1}{2} - \frac{\varepsilon_{M}}{6},\frac{1}{2} + \frac{\varepsilon_{M}}{6} \right\rbrack \right\rbrack \,\mathrm d \mathbb{P}_{Z}\left( z \right)} 
\\
&
= \int_{\frac{1}{2} - \frac{\varepsilon_{M}}{9}}^{\frac{1}{2} + \frac{\varepsilon_{M}}{9}}{\mathbb{P}\left\lbrack \left\{ \frac{1}{t}\sum_{k = 1}^{t}B_{z,k} - z < \frac{1}{2} - \frac{\varepsilon_{M}}{6} - z \right\} \cup \left\{ \frac{1}{t}\sum_{k = 1}^{t}B_{z,k} - z > \frac{1}{2} + \frac{\varepsilon_{M}}{6} - z \right\} \right\rbrack \,\mathrm d \mathbb{P}_{Z}\left( z \right)} 
\\
&
\leq \int_{\frac{1}{2} - \frac{\varepsilon_{M}}{9}}^{\frac{1}{2} + \frac{\varepsilon_{M}}{9}}{\left( \vphantom{\frac{\sum}{\displaystyle\sum}} \exp\left( - 2\left( \frac{1}{2} - \frac{\varepsilon_{M}}{6} - z \right)^{2}t \right) + \exp\left( - 2\left( \frac{1}{2} + \frac{\varepsilon_{M}}{6} - z \right)^{2}t \right) \right)\,\mathrm d \mathbb{P}_{Z}\left( z \right)}
\\
&
\leq \int_{\frac{1}{2} - \frac{\varepsilon_{M}}{9}}^{\frac{1}{2} + \frac{\varepsilon_{M}}{9}}{\left( \vphantom{\frac{\sum}{\displaystyle\sum}} \exp\left( - 2\left( \frac{1}{2} - \frac{\varepsilon_{M}}{6} - \frac{1}{2} + \frac{\varepsilon_{M}}{9} \right)^{2}t \right) + \exp\left( - 2\left( \frac{1}{2} + \frac{\varepsilon_{M}}{6} - \frac{1}{2} - \frac{\varepsilon_{M}}{9} \right)^{2}t \right) \right)\,\mathrm d \mathbb{P}_{Z}\left( z \right)} 
\\
&
= \frac{2}{9}\exp\left( - 2\left( \frac{\varepsilon_{M}}{18} \right)^{2}t \right) 
= \frac{2}{9}\exp\left( - \frac{\varepsilon_{M}^{2}}{162}t \right) = \frac{2}{9}\exp\left( - \frac{\left( \frac{7}{M} \right)^{2}}{162}t \right) = \frac{2}{9}\exp\left( - \frac{49}{162} \cdot \frac{t}{M^{2}} \right),
\end{align*}
}
where the first inequality follows from Hoeffding's inequality.
About the term \(\left( \text{IV} \right)\), we have
{\allowdisplaybreaks
\begin{align*}
\left( \text{IV} \right) 
& \leq \sqrt{\mathbb{E}\left\lbrack \left| Z - \frac{1}{t}\sum_{k = 1}^{t}B_{Z,k} \right|^{2}\mathbb{I}\left\{ Z \in \left\lbrack \frac{1}{2} - \frac{\varepsilon_{M}}{9},\frac{1}{2} + \frac{\varepsilon_{M}}{9} \right\rbrack \right\}\mathbb{I}\left\{ \frac{1}{t}\sum_{k = 1}^{t}B_{Z,k} \in \left\lbrack \frac{1}{2} - \frac{\varepsilon_{M}}{6},\frac{1}{2} + \frac{\varepsilon_{M}}{6} \right\rbrack\  \right\} \right\rbrack} \cdot
\\
&
\qquad
\cdot \sqrt{\mathbb{E}\left\lbrack \left| \frac{1}{t}\sum_{k = 1}^{t}B_{Z,k}-\mathbb{ E}\left\lbrack Z\mid B_{Z,1},\ldots,B_{Z,t} \right\rbrack \right|^{2}\mathbb{I}\left\{ Z \in \left\lbrack \frac{1}{2} - \frac{\varepsilon_{M}}{9},\frac{1}{2} + \frac{\varepsilon_{M}}{9} \right\rbrack \right\}\mathbb{I}\left\{ \frac{1}{t}\sum_{k = 1}^{t}B_{Z,k} \in \left\lbrack \frac{1}{2} - \frac{\varepsilon_{M}}{6},\frac{1}{2} + \frac{\varepsilon_{M}}{6} \right\rbrack\  \right\} \right\rbrack}
\\
&
\leq \sqrt{\mathbb{E}\left\lbrack \left| Z - \frac{1}{t}\sum_{k = 1}^{t}B_{Z,k} \right|^{2}\mathbb{I}\left\{ Z \in \left\lbrack \frac{1}{2} - \frac{\varepsilon_{M}}{9},\frac{1}{2} + \frac{\varepsilon_{M}}{9} \right\rbrack \right\} \right\rbrack} \cdot 
\\
&
\qquad\qquad
\cdot
\sqrt{\mathbb{E}\left\lbrack \left| \frac{1}{t}\sum_{k = 1}^{t}B_{Z,k}-\mathbb{ E}\left\lbrack Z\mid B_{Z,1},\ldots,B_{Z,t} \right\rbrack \right|^{2}\mathbb{I}\left\{ \frac{1}{t}\sum_{k = 1}^{t}B_{Z,k} \in \left\lbrack \frac{1}{2} - \frac{\varepsilon_{M}}{6},\frac{1}{2} + \frac{\varepsilon_{M}}{6} \right\rbrack\  \right\} \right\rbrack}
\\
&
= \sqrt{\frac{4}{147} \cdot \frac{1}{t}} \cdot \sqrt{\mathbb{E}\left\lbrack \left| \frac{1}{t}\sum_{k = 1}^{t}B_{Z,k}-\mathbb{ E}\left\lbrack Z\mid B_{Z,1},\ldots,B_{Z,t} \right\rbrack \right|^{2}\mathbb{I}\left\{ \frac{1}{t}\sum_{k = 1}^{t}B_{Z,k} \in \left\lbrack \frac{1}{2} - \frac{\varepsilon_{M}}{6},\frac{1}{2} + \frac{\varepsilon_{M}}{6} \right\rbrack\  \right\} \right\rbrack} \eqqcolon \left( \circ \right),
\end{align*}%
}%
where the first inequality follows from Cauchy-Schwarz and the last
inequality follows from \(\left( \star \right)\).
Now, using that \(\left( a - b \right)^{2} \leq 2a^{2} + 2b^{2}\) for
any \(a,b\in\mathbb{ R}\), we get:
\begin{align*}
&
\left| \frac{1}{t}\sum_{k = 1}^{t}B_{Z,k}-\mathbb{ E}\left\lbrack Z\mid B_{Z,1},\ldots,B_{Z,t} \right\rbrack \right|^{2}\mathbb{I}\left\{ \frac{1}{t}\sum_{k = 1}^{t}B_{Z,k} \in \left\lbrack \frac{1}{2} - \frac{\varepsilon_{M}}{6},\frac{1}{2} + \frac{\varepsilon_{M}}{6} \right\rbrack\  \right\} 
\\
&
\qquad
\leq 2\left| \frac{1}{t}\sum_{k = 1}^{t}B_{Z,k} - \frac{1 + \sum_{k = 1}^{t}B_{Z,k}}{t + 2} \right|^{2} + 
\\
&
\qquad\qquad
2\left| \frac{1 + \sum_{k = 1}^{t}B_{Z,k}}{t + 2}-\mathbb{ E}\left\lbrack Z\mid B_{Z,1},\ldots,B_{Z,t} \right\rbrack \right|^{2}\mathbb{I}\left\{ \frac{1}{t}\sum_{k = 1}^{t}B_{Z,k} \in \left\lbrack \frac{1}{2} - \frac{\varepsilon_{M}}{6},\frac{1}{2} + \frac{\varepsilon_{M}}{6} \right\rbrack\  \right\} 
\eqqcolon
\left( \mathrm V \right) + \left( \mathrm{VI} \right).
\end{align*}
Simple calculations show that
\[\left( \mathrm V \right) \leq \frac{18}{t^{2}}.\]
About \(\left( \text{VI} \right)\), we first compute
\(\mathbb{E}\left\lbrack Z\mid B_{Z,1},\ldots,B_{Z,t} \right\rbrack\)
using Bayes' formula and get
\[\mathbb{E}\left\lbrack Z\mid B_{Z,1},\ldots,B_{Z,t} \right\rbrack = \frac{\int_{\left\lbrack \frac{1}{2} - \varepsilon_{M},\frac{1}{2} + \varepsilon_{M} \right\rbrack}^{}{p^{1 + \sum_{k = 1}^{t}B_{Z,k}}\left( 1 - p \right)^{t - \sum_{k = 1}^{t}B_{Z,k}}\,\mathrm{d}p}\ }{\int_{\left\lbrack \frac{1}{2} - \varepsilon_{M},\frac{1}{2} + \varepsilon_{M} \right\rbrack}^{}{p^{\sum_{k = 1}^{t}B_{Z,k}}\left( 1 - p \right)^{t - \sum_{k = 1}^{t}B_{Z,k}}\,\mathrm{d}p}},\]
then, we select, for any \(n\in\mathbb{ N}\) and
\(x \in \left( 0,1 \right)\), a binomial random variable
\(\operatorname{Bin}\left( n,x \right)\) of parameters \(n\) and \(x\),
to get
{\allowdisplaybreaks\small
\begin{align*}
&\left( \text{VI} \right) 
= 2\left| \frac{1 + \sum_{k = 1}^{t}B_{Z,k}}{t + 2} - \frac{\int_{\left\lbrack \frac{1}{2} - \varepsilon_{M},\frac{1}{2} + \varepsilon_{M} \right\rbrack}^{}{p^{1 + \sum_{k = 1}^{t}B_{Z,k}}\left( 1 - p \right)^{t - \sum_{k = 1}^{t}B_{Z,k}}\,\mathrm{d}p}\ }{\int_{\left\lbrack \frac{1}{2} - \varepsilon_{M},\frac{1}{2} + \varepsilon_{M} \right\rbrack}^{}{p^{\sum_{k = 1}^{t}B_{Z,k}}\left( 1 - p \right)^{t - \sum_{k = 1}^{t}B_{Z,k}}\,\mathrm{d}p}} \right|^{2}\mathbb{I}\left\{ \frac{\sum_{k = 1}^{t}B_{Z,k}}{t} \in \left\lbrack \frac{1}{2} - \frac{\varepsilon_{M}}{6},\frac{1}{2} + \frac{\varepsilon_{M}}{6} \right\rbrack\  \right\}
\\
&
= \left\lbrack 2\left| \frac{s + 1}{t + 2}\frac{t - s + 1}{t + 2}\frac{\left| \mathbb{P}\left\lbrack \operatorname{Bin}\left( t + 2,\frac{1}{2} + \varepsilon_{M} \right) = s + 1 \right\rbrack-\mathbb{ P}\left\lbrack \operatorname{Bin}\left( t + 2,\frac{1}{2} - \varepsilon_{M} \right) = s + 1 \right\rbrack \right|}{\left| \mathbb{P}\left\lbrack \operatorname{Bin}\left( t + 1,\frac{1}{2} + \varepsilon_{M} \right) \geq s + 1 \right\rbrack-\mathbb{ P}\left\lbrack \operatorname{Bin}\left( t + 1,\frac{1}{2} - \varepsilon_{M} \right) \geq s + 1 \right\rbrack \right|} \right|^{2}\mathbb{I}\left\{ \frac{s}{t} \in \left\lbrack \frac{1}{2} - \frac{\varepsilon_{M}}{6},\frac{1}{2} + \frac{\varepsilon_{M}}{6} \right\rbrack\  \right\} \right\rbrack_{|s = \sum_{k = 1}^{t}B_{Z,k}}
\\
&
\leq \left\lbrack 2\left| \frac{\left| \mathbb{P}\left\lbrack \operatorname{Bin}\left( t + 2,\frac{1}{2} + \varepsilon_{M} \right) = s + 1 \right\rbrack-\mathbb{ P}\left\lbrack \operatorname{Bin}\left( t + 2,\frac{1}{2} - \varepsilon_{M} \right) = s + 1 \right\rbrack \right|}{\left| \mathbb{P}\left\lbrack \operatorname{Bin}\left( t + 1,\frac{1}{2} + \varepsilon_{M} \right) \geq s + 1 \right\rbrack-\mathbb{ P}\left\lbrack \operatorname{Bin}\left( t + 1,\frac{1}{2} - \varepsilon_{M} \right) \geq s + 1 \right\rbrack \right|} \right|^{2}\mathbb{I}\left\{ \frac{s}{t} \in \left\lbrack \frac{1}{2} - \frac{\varepsilon_{M}}{6},\frac{1}{2} + \frac{\varepsilon_{M}}{6} \right\rbrack\  \right\} \right\rbrack_{|s = \sum_{k = 1}^{t}B_{Z,k}}
\\
&
\leq \left\lbrack 2\left| \frac{\max\left( \mathbb{P}\left\lbrack \operatorname{Bin}\left( t + 2,\frac{1}{2} + \varepsilon_{M} \right) = s + 1 \right\rbrack\mathbb{,P}\left\lbrack \operatorname{Bin}\left( t + 2,\frac{1}{2} - \varepsilon_{M} \right) = s + 1 \right\rbrack \right)}{\left| \mathbb{P}\left\lbrack \operatorname{Bin}\left( t + 1,\frac{1}{2} + \varepsilon_{M} \right) \geq s + 1 \right\rbrack-\mathbb{ P}\left\lbrack \operatorname{Bin}\left( t + 1,\frac{1}{2} - \varepsilon_{M} \right) \geq s + 1 \right\rbrack \right|} \right|^{2}\mathbb{I}\left\{ \frac{s}{t} \in \left\lbrack \frac{1}{2} - \frac{\varepsilon_{M}}{6},\frac{1}{2} + \frac{\varepsilon_{M}}{6} \right\rbrack\  \right\} \right\rbrack_{|s = \sum_{k = 1}^{t}B_{Z,k}}
\\
&
\leq \left\lbrack 2\left| \frac{\max\left( \mathbb{P}\left\lbrack \operatorname{Bin}\left( t + 2,\frac{1}{2} + \varepsilon_{M} \right) \leq s + 1 \right\rbrack\mathbb{,P}\left\lbrack \operatorname{Bin}\left( t + 2,\frac{1}{2} - \varepsilon_{M} \right) \geq s + 1 \right\rbrack \right)}{\left| \mathbb{P}\left\lbrack \operatorname{Bin}\left( t + 1,\frac{1}{2} + \varepsilon_{M} \right) \geq s + 1 \right\rbrack-\mathbb{ P}\left\lbrack \operatorname{Bin}\left( t + 1,\frac{1}{2} - \varepsilon_{M} \right) \geq s + 1 \right\rbrack \right|} \right|^{2}\mathbb{I}\left\{ \frac{s}{t} \in \left\lbrack \frac{1}{2} - \frac{\varepsilon_{M}}{6},\frac{1}{2} + \frac{\varepsilon_{M}}{6} \right\rbrack\  \right\} \right\rbrack_{|s = \sum_{k = 1}^{t}B_{Z,k}} \eqqcolon \left( \varheartsuit \right).
\end{align*}
}%
Now, since, for any \(s,t\in\mathbb{ N}\), if
\(\frac{s}{t} \in \left\lbrack \frac{1}{2} - \frac{\varepsilon_{M}}{6},\frac{1}{2} + \frac{\varepsilon_{M}}{6} \right\rbrack\)
and \(t \geq \frac{6}{7}M\) we have that
\[\frac{s + 1}{t + 1} \in \left\lbrack \frac{1}{2} - \frac{\varepsilon_{M}}{3},\frac{1}{2} + \frac{\varepsilon_{M}}{3} \right\rbrack,\ \ \frac{s + 1}{t + 2} \in \left\lbrack \frac{1}{2} - \frac{\varepsilon_{M}}{3},\frac{1}{2} + \frac{\varepsilon_{M}}{3} \right\rbrack\]
we get, using Hoeffding inequality in each of the following
inequalities, that
\begin{align}
&
\mathbb{P}\left\lbrack \operatorname{Bin}\left( t + 1,\frac{1}{2} + \varepsilon_{M} \right) \geq s + 1 \right\rbrack=\mathbb{ P}\left\lbrack \frac{1}{t + 1}\operatorname{Bin}\left( t + 1,\frac{1}{2} + \varepsilon_{M} \right) - \left( \frac{1}{2} + \varepsilon_{M} \right) \geq \frac{s + 1}{t + 1} - \left( \frac{1}{2} + \varepsilon_{M} \right) \right\rbrack 
\nonumber
\\
&
\qquad = 1 - \mathbb{P}\left\lbrack \frac{1}{t + 1}\operatorname{Bin}\left( t + 1,\frac{1}{2} + \varepsilon_{M} \right) - \left( \frac{1}{2} + \varepsilon_{M} \right) < - \left( \left( \frac{1}{2} + \varepsilon_{M} \right) - \frac{s + 1}{t + 1} \right) \right\rbrack 
\nonumber
\\
&
\qquad \geq 1 - \exp\left( - 2\left( \left( \frac{1}{2} + \varepsilon_{M} \right) - \frac{s + 1}{t + 1} \right)^{2}\left( t + 1 \right) \right) \geq 1 - \exp\left( - \frac{8}{9}\varepsilon_{M}^{2}\left( t + 1 \right) \right)
\label{eq:uno}
\end{align}
while
\begin{multline}
\mathbb{P}\left\lbrack \operatorname{Bin}\left( t + 1,\frac{1}{2} - \varepsilon_{M} \right) \geq s + 1 \right\rbrack=\mathbb{ \ P}\left\lbrack \frac{1}{t + 1}\operatorname{Bin}\left( t + 1,\frac{1}{2} - \varepsilon_{M} \right) - \left( \frac{1}{2} - \varepsilon_{M} \right) \geq \frac{s + 1}{t + 1} - \left( \frac{1}{2} - \varepsilon_{M} \right) \right\rbrack 
\\
\leq \exp\left( - 2\left( \frac{s + 1}{t + 1} - \left( \frac{1}{2} - \varepsilon_{M} \right) \right)^{2}\left( t + 1 \right) \right) \leq \exp\left( - \frac{8}{9}\varepsilon_{M}^{2}\left( t + 1 \right) \right)
\label{eq:due}
\end{multline}
and
\begin{multline}
\mathbb{P}\left\lbrack \operatorname{Bin}\left( t + 2,\frac{1}{2} + \varepsilon_{M} \right) \leq s + 1 \right\rbrack = \mathbb{P}\left\lbrack \frac{1}{t + 2}\operatorname{Bin}\left( t + 2,\frac{1}{2} + \varepsilon_{M} \right) - \left( \frac{1}{2} + \varepsilon_{M} \right) \leq \frac{s + 1}{t + 2} - \left( \frac{1}{2} + \varepsilon_{M} \right) \right\rbrack 
\\
\leq \exp\left( - 2\left( \frac{s + 1}{t + 2} - \left( \frac{1}{2} + \varepsilon_{M} \right) \right)^{2}\left( t + 2 \right) \right) \leq \exp\left( - \frac{8}{9}\varepsilon_{M}^{2}\left( t + 2 \right) \right)
\label{eq:tre}
\end{multline}
and, finally
\begin{multline}
\mathbb{P}\left\lbrack \operatorname{Bin}\left( t + 2,\frac{1}{2} - \varepsilon_{M} \right) \geq s + 1 \right\rbrack=\mathbb{ P}\left\lbrack \frac{1}{t + 2}\operatorname{Bin}\left( t + 2,\frac{1}{2} - \varepsilon_{M} \right) - \left( \frac{1}{2} - \varepsilon_{M} \right) \geq \frac{s + 1}{t + 2} - \left( \frac{1}{2} - \varepsilon_{M} \right) \right\rbrack
\\
\leq \exp\left( - 2\left( \frac{s + 1}{t + 2} - \left( \frac{1}{2} - \varepsilon_{M} \right) \right)^{2}\left( t + 2 \right) \right) \leq \exp\left( - \frac{8}{9}\varepsilon_{M}^{2}\left( t + 2 \right) \right). 
\label{eq:quattro}
\end{multline}
Plugging the inequalities
\eqref{eq:uno}, \eqref{eq:due}, \eqref{eq:tre}, \eqref{eq:quattro}
into \(\left( \varheartsuit \right)\), we get
\[\left( \varheartsuit \right) \leq 2\left( \frac{\exp\left( - \frac{8}{9}\varepsilon_{M}^{2}\left( t + 2 \right) \right)}{1 - 2\exp\left( - \frac{8}{9}\varepsilon_{M}^{2}\left( t + 1 \right) \right)} \right)^{2}\mathbb{I}\left\{ \frac{1}{t}\sum_{k = 1}^{t}B_{Z,k} \in \left\lbrack \frac{1}{2} - \frac{\varepsilon_{M}}{6},\frac{1}{2} + \frac{\varepsilon_{M}}{6} \right\rbrack\  \right\} \leq 2\left( \frac{\exp\left( - \frac{8}{9}\varepsilon_{M}^{2}\left( t + 2 \right) \right)}{1 - 2\exp\left( - \frac{8}{9}\varepsilon_{M}^{2}\left( t + 1 \right) \right)} \right)^{2}\]
and hence
\begin{align*}
&
\left( \text{IV} \right) \leq \left( \circ \right) \leq \sqrt{\frac{4}{147} \cdot \frac{1}{t}} \cdot \sqrt{\mathbb{E}\left\lbrack \left( \mathrm V \right) + \left( \text{VI} \right) \right\rbrack} \leq \sqrt{\frac{4}{147} \cdot \frac{1}{t}} \cdot \sqrt{\frac{18}{t^{2}} + 2\left( \frac{\exp\left( - \frac{8}{9}\varepsilon_{M}^{2}\left( t + 2 \right) \right)}{1 - 2\exp\left( - \frac{8}{9}\varepsilon_{M}^{2}\left( t + 1 \right) \right)} \right)^{2}} 
\\
&
\qquad
= \sqrt{\frac{4}{147}}\sqrt{18 + 2\left( t\frac{\exp\left( - \frac{392}{9}\frac{t + 2}{M^{2}} \right)}{1 - 2\exp\left( - \frac{392}{9}\frac{t + 2}{M^{2}} \right)} \right)^{2}} \cdot \frac{1}{t^{3/2}}
\end{align*}
Putting everything together, we have:
\begin{multline*}    
\mathbb{E}\left\lbrack \left( Z-\mathbb{ E}\left\lbrack Z\mid B_{Z,1},\ldots,B_{Z,t} \right\rbrack \right)^{2} \right\rbrack \geq \left( \mathrm I \right) - 2 \cdot \left( \text{II} \right) \geq \frac{4}{147} \cdot \frac{1}{t} - 2 \cdot \brb{ \left( \text{III} \right) + \left( \text{IV} \right) } 
\\
\geq \frac{4}{147} \cdot \frac{1}{t} - 2 \cdot \left( \frac{2}{9}\exp\left( - \frac{49}{162} \cdot \frac{t}{M^{2}} \right) + \sqrt{\frac{4}{147}}\sqrt{18 + 2\left( t\frac{\exp\left( - \frac{392}{9}\frac{t + 2}{M^{2}} \right)}{1 - 2\exp\left( - \frac{392}{9}\frac{t + 2}{M^{2}} \right)} \right)^{2}} \cdot \frac{1}{t^{3/2}} \right) \eqqcolon \left( \spadesuit \right)
\end{multline*}
Elementary computations show that:
\begin{itemize}
\item
  if \(t \geq 274M^{4}\) then
  \(\frac{4}{9}\exp\left( - \frac{49}{162} \cdot \frac{t}{M^{2}} \right) \leq \frac{1}{147} \cdot \frac{1}{t}\)
\item
  if \(t \geq 2M^{4}\) then
  \(\exp\left( - \frac{392}{9}\frac{t + 2}{M^{2}} \right) \leq \frac{1}{t}\)
\item
  if \(t \geq \frac{4}{100}M^{2}\) then
  \(1 - 2\exp\left( - \frac{392}{9}\frac{t + 2}{M^{2}} \right) \geq \frac{1}{2}\)
\item
  therefore, if \(t \geq 2M^{4}\) and \(t \geq \frac{4}{100}M^{2}\) and
  \(t \geq 61152\) then
  \(2\sqrt{\frac{4}{147}}\sqrt{18 + 2\left( t\frac{\exp\left( - \frac{392}{9}\frac{t + 2}{M^{2}} \right)}{1 - 2\exp\left( - \frac{392}{9}\frac{t + 2}{M^{2}} \right)} \right)^{2}} \cdot \frac{1}{t^{\frac{3}{2}}} \leq \frac{1}{147} \cdot \frac{1}{t}.\)
\end{itemize}
These together with $( \spadesuit )$ implies that, if \(t \geq \max\left( 274M^{4},61152 \right)\), which
is in particular implied by \(t \geq 580M^{4}\), we have that
\[\mathbb{E}\left\lbrack \left( Z-\mathbb{ E}\left\lbrack Z\mid B_{Z,1},\ldots,B_{Z,t} \right\rbrack \right)^{2} \right\rbrack \geq \frac{2}{147} \cdot \frac{1}{t},\]
which implies that if \(2\left( t - 1 \right) \geq 580M^{4}\), which is
implied by \(t \geq 580M^{4}\),
\[\mathbb{E}\left\lbrack \left( Z-\mathbb{ E}\left\lbrack Z\mid B_{Z,1},\ldots,B_{Z,2\left( t - 1 \right)} \right\rbrack \right)^{2} \right\rbrack \geq \frac{1}{147} \cdot \frac{1}{t - 1}.\]

\section{Missing details in the proof of Theorem~\ref{t:lower-bound-full-general}}

    For each $\e \in \bsb{ -\frac{1}{4},\frac{1}{4} }$, consider the distribution
    \[
        \nu_\e \coloneqq \frac{1}{4}\delta_0 + \lrb{\frac{1}{4} + \e}\delta_{1/3} + \lrb{\frac{1}{4} - \e}\delta_{2/3} + \frac{1}{4}\delta_{1} \;.
    \]
    Consider for each $\e \in \bsb{ -\frac{1}{4},\frac{1}{4} }$ an i.i.d.\ sequence $(B_{\e,t})_{t \in \N}$ of Bernoulli random variables of parameter $\frac{1}{2}+2\e$, and consider two i.i.d.\ sequences $(B_t)_{t \in \N}, (\tilde{B}_t)_{t \in \N}$ of parameter $1/2$, such that the family of random variables $\lrb{ (B_{\e,t})_{\e\in [-\frac{1}{4},\frac{1}{4}], t \in \N}, (B_t)_{t \in \N}, (\tilde{B}_t)_{t \in \N} }$ is an independent family.
    For each $t \in \N$ and each $\e \in [-\frac{1}{4},\frac{1}{4}]$, define
    \begin{equation}
        V_{\e,t} \coloneqq \frac{1}{3} (1-B_t)B_{\e,t}+\frac{2}{3}(1-B_t)(1-B_{\e,t})+B_t\tilde{B}_t \;,    
    \end{equation}
    and, for each $\e \in \bsb{ -\frac{1}{4},\frac{1}{4} }$, notice that $(V_{\e,t})_{t \in \N}$ is an i.i.d\ sequence with common distribution $\nu_\e$.
    For any $\e\in\bsb{ -\frac{1}{4},\frac{1}{4} }$, $p\in[0,1]$, and $t\in \N$, let $\GFT_{\e,t}(p) \coloneqq \gft(p,V_{\e,2t-1}, V_{\e,2t})$.
    For each $\e \in \bsb{ -\frac{1}{8},\frac{1}{8} }$ and each $t \in \N$, note that:
    \begin{align}    
        \label{e:opt_1}
        &
        \max_{p \in \lcb{\frac{1}{3},\frac{2}{3}}} \E\bsb{\GFT_{\e,t}(p)} = \max_{p \in [0,1]} \E\bsb{\GFT_{\e,t}(p)}
        \\
        \label{e:opt_2}
        &
        \min_{p \in \lcb{\frac{1}{3},\frac{2}{3}}} \E\bsb{\GFT_{\e,t}(p)} - \max_{p \in [0,1]\m \lcb{\frac{1}{3},\frac{2}{3}}} \E\bsb{\GFT_{\e,t}(p)} = \Omega(1)
        \\
        \label{e:opt_3}
        &
        \E\bsb{\GFT_{\e,t}(\fracc{1}{3})} - \E\bsb{\GFT_{\e,t}(\fracc{2}{3})} =  \operatorname{sgn}(\e) \cdot \Omega\brb{\labs{\e}} 
    \end{align}
    Fix a time horizon $T \in \N$ and select $\e \coloneqq T^{-1/2}$.
    We will show that for each algorithm for the full-feedback setting and each time horizon $T$, if $R_T^{\nu}$ is the regret of the algorithm at time horizon $T$ when the underlying distribution of the traders' valuations is $\nu$, then $\max\brb{ R_T^{\nu_{-\e}}, R_T^{\nu_{+\e}} } = \Omega\brb{ \sqrt{T} }$.
    Notice that, by posting prices in the wrong region $[0,1]\m\{1/3\}$ (resp., $[0,1]\m\{2/3\}$) in the $+\e$ (resp., $-\e$) case, the learner incurs a $\Omega(\e) = \Omega\brb{1/\sqrt{T}}$ instantaneous regret by \eqref{e:opt_1}, \eqref{e:opt_2}, and \eqref{e:opt_3}.
    Then, in order to attempt suffering less than $\Omega\brb{1/\sqrt{T} \cdot T} = \Omega\brb{\sqrt{T }}$ regret, the algorithm would have to detect the sign of $\pm\e$ and play accordingly.
    However, the algorithm has no means to gather enough information to accomplish this task in due time. 
    In fact, notice that the feedback received from the two traders at time $t$ after having posted a price $p$ is $V_{\pm\e,2t-1}$ and $V_{\pm\e,2t}$, which can't give more information about $\pm\e$ than the information carried by the two Bernoullis $B_{\pm\e,2t-1}$ and $B_{\pm\e,2t}$. 
    Since (via an information-theoretic argument) in order to distinguish the sign of $\pm\e$ having access to i.i.d.\ Bernoulli random variables of parameter $\frac{1}{2}\pm 2\e$ requires $\Omega(1/\e^2) = \Omega( T )$ samples, the algorithm will have already suffered a regret $\Omega\brb{ T } \cdot \Omega(1/\sqrt{T}) = \Omega\brb{ \sqrt{T}}$ before having the chance to distinguish the sign of $\pm\e$.

\end{document}